\documentclass[lettersize,journal]{IEEEtran}
\usepackage{amsmath,amsfonts,amsthm}
\usepackage{algorithmic,algorithm}
\usepackage{array}
\usepackage[caption=false,font=normalsize,labelfont=sf,textfont=sf]{subfig}
\usepackage{textcomp}
\usepackage{stfloats}
\usepackage{url}
\usepackage{verbatim}
\usepackage{graphicx}
\def\BibTeX{{\rm B\kern-.05em{\sc i\kern-.025em b}\kern-.08em
    T\kern-.1667em\lower.7ex\hbox{E}\kern-.125emX}}
\usepackage{balance}
\usepackage{xcolor}

\newtheorem{notation}{Notation}
\newtheorem{remark}{Remark}
\newtheorem{assumption}{Assumption}
\newtheorem{corollary}{Corollary}
\newtheorem{theorem}{Theorem}
\newtheorem{definition}{Definition}
\newtheorem{lemma}{Lemma}
\newtheorem{example}{Example}

\newcommand{\change}[1]{{\textcolor{black}{#1}}}
\newcommand{\secchange}[1]{{\textcolor{black}{#1}}}

\begin{document}

\title{Efficient-Adam: Communication-Efficient Distributed Adam}
\author{Congliang Chen, Li Shen, Wei Liu, \IEEEmembership{Fellow, IEEE}, 
Zhi-Quan Luo, \IEEEmembership{Fellow, IEEE}
\thanks{This work is supported in part by the Major Science and Technology Innovation 2030 “Brain Science and Brain-like Research” key project (No. 2021ZD0201405), the National Natural Science Foundation of China (No. 61731018) and the Guangdong Provincial Key Laboratory of Big Data Computing.}
\thanks{Congliang Chen is with the Chinese University of Hong Kong (Shenzhen), Shenzhen, Guangdong, China (e-mail: congliangchen@link.cuhk.edu.cn). }
\thanks{Li Shen is with JD Explore Academy, Beijing, China 
(e-mail: mathshenli@gmail.com).}
\thanks{Wei Liu is with Tencent, Shenzhen, Guangdong, China (e-mail: wl2223@columbia.edu).}
\thanks{Zhi-Quan Luo is with the Chinese University of  Hong Kong (Shenzhen), Shenzhen, Guangdong, China, and Shenzhen Research Institute of Big Data, Shenzhen, Guangdong, China (e-mail: luozq@cuhk.edu.cn).}
\thanks{Corresponding author: Li Shen.}
\thanks{Manuscript received November 23, 2022, revised May 26, 2023}}

\markboth{Journal of \LaTeX\ Class Files,~Vol.~14, No.~8, August~2021}%
{Shell \MakeLowercase{\textit{et al.}}: A Sample Article Using IEEEtran.cls for IEEE Journals}


\maketitle
\begin{abstract}\label{abstract-sec}
Distributed adaptive stochastic gradient methods have been widely used for large-scale nonconvex optimization, such as training deep learning models. However, their communication complexity on finding $\varepsilon$-stationary points has rarely been analyzed in the nonconvex setting. In this work, we present a novel communication-efficient distributed Adam in the parameter-server model for stochastic nonconvex optimization, dubbed {\em Efficient-Adam}. Specifically, we incorporate a two-way quantization scheme into Efficient-Adam to reduce the communication cost between the workers and server. Simultaneously, we adopt a two-way error feedback strategy to reduce the biases caused by the two-way quantization on both the server and workers, respectively. In addition, we establish the iteration complexity for the proposed Efficient-Adam with a class of quantization operators, and further characterize its communication complexity between the server and workers when an $\varepsilon$-stationary point is achieved. 
Finally, we apply Efficient-Adam to solve a toy stochastic convex optimization problem and train deep learning models on real-world vision and language tasks. Extensive experiments together with a theoretical guarantee justify the merits of Efficient Adam.
\end{abstract}


\section{Introduction}\label{sec:introduction}

\IEEEPARstart{L}{et}   $\mathbb{X}$ be a finite-dimensional linear vector space.  We focus on a stochastic optimization problem:
\begin{equation}\label{problem}
\min_{x\in \mathbb{X}}\ F(x) = \mathbb{E}_{\xi\sim\mathbb{P}} [f(x,\xi)], 
\end{equation}
where $F\!:\mathbb{X}\!\to \mathbb{R}$ is a proper, lower semi-continuous smooth function that could be nonconvex, and $\xi$ is a random variable with an unknown distribution $\mathbb{P}$. Stochastic optimizations commonly appear in the fields of statistical machine learning \cite{bishop2006pattern}, deep learning \cite{goodfellow2016deep}, and reinforcement learning  \cite{sutton2018reinforcement}, which include sparse learning  \cite{hastie2015statistical}, representation learning  \cite{bengio2013representation}, classification \cite{deng2009imagenet}, regression, etc. 

Due to the existence of expectation over an unknown distribution, the population risk $ \mathbb{E}_{\xi\sim\mathbb{P}} [f(x,\xi)]$ with a complicated function $F$ is approximated via an empirical risk function $\frac{1}{n}\sum_{i=1}^{n} [f_{i}(x,\xi_{i})]$ via a large number of samples $\{\xi_{i}\}$. For example, we may use the ImangeNet dataset for the image classification task \cite{deng2009imagenet}, the COCO dataset for the object detection task \cite{lin2014microsoft}, and the GLUE dataset for the natural language understanding task \cite{wang2018glue}, respectively. As all of the above datasets are large, we have to optimize problem \eqref{problem} via a large number of samples. Meanwhile, to learn appropriate distributions for different tasks, it is hard to design the exact mathematical formulation of $f$. Thus, an effective way is to set a complicated deep neural network as a surrogate function with a large number of parameters/FLOPs, making problem \eqref{problem} an ultra-large-scale nonconvex stochastic optimization \cite{shen2023efficient}. 

However, it could be impossible to train a deep neural network with a large number of parameters over a large-scale dataset within a single machine. Fortunately, we have some promising approaches to tackle this problem. One of them is to extend the stochastic gradient descent (SGD) method to \change{a} distributed version \cite{li2014efficient}. Then using the distributed SGD method, we train a deep learning model with multiple machines in a distributed mode \cite{goyal2017accurate,you2019large}. However, for the vanilla SGD method, how to tune a suitable learning rate for different tasks remains challenging.  This dilemma is more serious for distributed SGD methods since there are multiple learning rates needed to be tuned for multiple machines. 
Moreover, the communication overhead is another issue in distributed methods. How to reduce the communication cost among multiple machines is also challenging.
Recently, Hou et al.\cite{hou2018analysis} proposed a distributed Adam  \cite{kingma2014adam} with weights and gradients being quantized to reduce the communication cost between the workers and the server. However, their theoretical analysis \change{is} merely restricted to the convex setting, and they didn't provide the bit-communication complexity either, which hampers the potential applications. In addition, both the weights and gradients quantization techniques will introduce additional errors, which may degrade the performance of the vanilla Adam optimizer. 
On the other hand, distributed Adam has already been built in several deep learning platforms, such as PyTorch \cite{paszke2019pytorch}, TensorFlow \cite{abadi2016tensorflow}, and MXNet \cite{chen2015mxnet}, and it has since been broadly used for training deep learning models. However, its communication complexity under the distributed mode has rarely been analyzed in either convex or nonconvex settings.

\begin{figure*}[!htbp]
   \centering
   \includegraphics[height=5.5cm, width=0.8\textwidth]{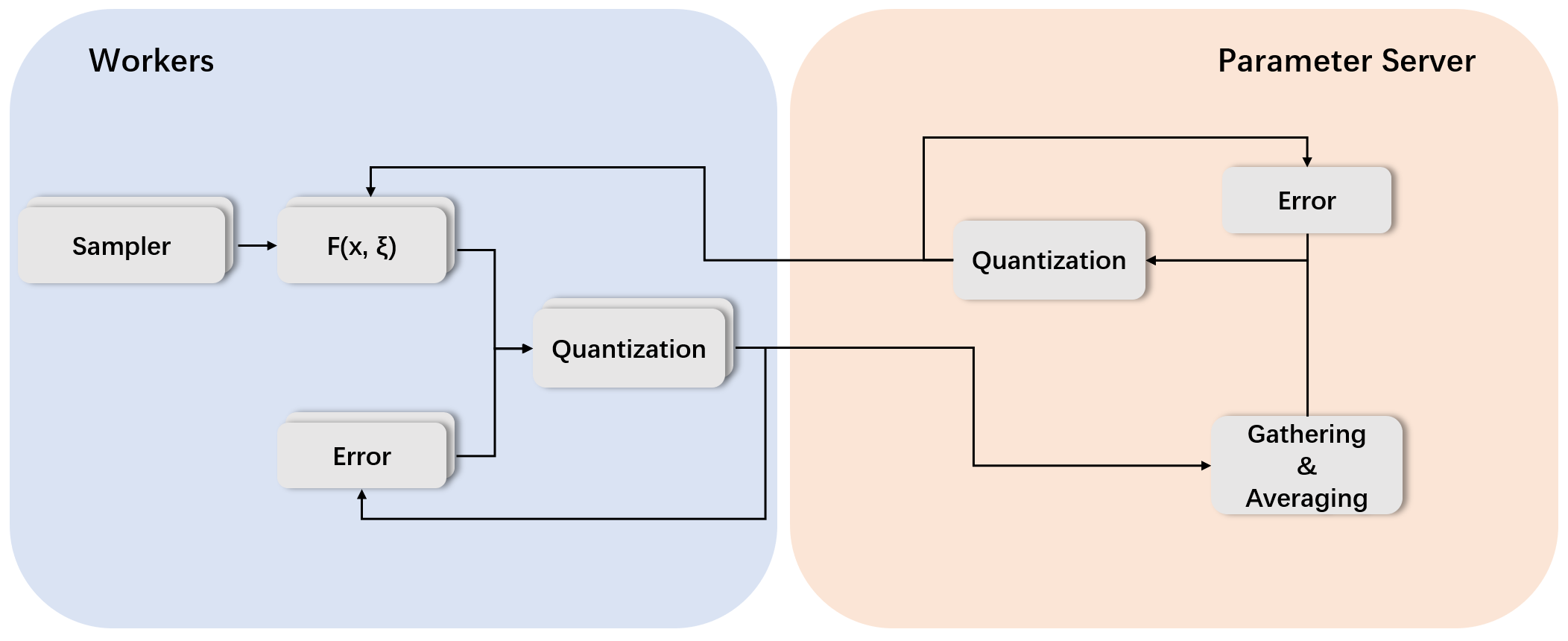}
%   \vspace{-0.3cm}
    \caption{The workflow of Efficient-Adam in the Parameter-Server model. For each worker, parameters are updated via Adam. For the server, information is gathered and averaged to update parameters. There exists a two-way quantization scheme to reduce the communication cost between each worker and the server. Moreover, a two-way error feedback strategy compensates for quantization errors for both the workers and the server.}
    \label{parameter-server}
%    \vspace{-0.5cm}
\end{figure*}

In this work, we propose a communication-efficient distributed adaptive stochastic gradient descent method by incorporating a two-way quantization of updates and a two-way error feedback strategy, dubbed {\rm Efficient-Adam}, to solve the nonconvex stochastic problem \eqref{problem}. As it is illustrated in Figure \ref{parameter-server}, the architecture of Efficient-Adam belongs to a Parameter-Server \cite{smola2010architecture,li2014scaling} distributed system. For each worker, in each iteration, we first sample a stochastic gradient of problem \eqref{problem}. Then, we calculate the updates with the Adam optimizer. Next, we quantize the updates based on a specifically designed quantization mapping and an error term and send the quantized updates to the server. \change{After that, we update the local error terms using the updated quantization mapping and the new error terms.}
%After that, we update local error terms with the updates, quantization mapping, and error terms. 
At last, we receive the averaged updates from the server and update local iterates. On the other hand, in the server, we first gather and average the compressed updates from each worker. Then, we quantize the average updates and broadcast them to each worker. Next, we update iterates and error feedback terms. From the workflow of Efficient-Adam in Figure \ref{parameter-server}, it can be seen that communicated bi-directional information is all quantized in advance. With the proper quantization mapping, communication costs can be reduced largely. In addition, error terms in the workers and server can compensate for errors that are introduced by the two-way quantization steps between all the workers and the server, which \change{helps} accelerate the convergence of Efficient-Adam for problem \eqref{problem}.  

Moreover, we explore the iteration complexity for Efficient-Adam with a class of quantization operators when an $\varepsilon$-stationary point is achieved. For a carefully designed quantization mapping, we further characterize its overall communication complexity in terms of bits between the server and workers. On the other hand, when the quantization mapping is generalized to a compressor as \cite{stich2019error,stich2018sparsified,zheng2019communication}, Efficient-Adam can enjoy the same convergence rate as \change{the} full-precision Adam in Zou et al. \cite{zou2019sufficient} and Chen et al. \cite{chen2022towards}.
Experimentally, we apply Efficient-Adam to train deep learning models on computer vision and natural language processing tasks to demonstrate its efficacy. To the end, we summarize our contribution in four-fold:
%\vspace{-0.1cm}
\begin{itemize}
    \item We propose a communication-efficient distributed Adam to solve stochastic  problem \eqref{problem}. We dub it {\em Efficient-Adam} which utilizes a two-way quantization scheme to reduce the communication cost and a two-way error feedback strategy to compensate for quantization errors. 
    \item We characterize the iteration complexity of Efficient-Adam in the non-convex setting. Under proper assumptions, we further characterize its overall communication complexity in terms of bits between the server and workers via a specifically designed quantization mapping. 
    \item We show the convergence rate of Efficient-Adam can further be improved to the same order of vanilla Adam, i.e., $\mathcal{O}(1/\sqrt{T})$, once we generalize the quantization mapping in Efficient-Adam to a compressor mapping as \cite{stich2019error,stich2018sparsified}. 
    \item We conduct experiments on computer vision and natural language processing tasks to demonstrate the efficacy of the proposed Efficient-Adam, as well as the related two-way quantization and  two-way  error feedback techniques.
\end{itemize}

%\vspace{-0.2cm}
\section{Related Work}\label{related-work-sec}
\begin{table*}[htb!]
    \centering
    \caption{Comparison among different methods. Note that Efficient-Adam can support both the compression and quantization mappings.}
    %Particularly, \citet{zheng2019communication}  merely supported compressor mapping while Efficient-Adam can support both the compressor and quantization mappings.}
%    \vspace{-0.2cm}
        \renewcommand{\arraystretch}{1.2}
    \begin{tabular}{|c|c|c|c|c|c|c|}
    \hline
     \!Method\! & \!non-convex\!&\!\! adaptive lr\! &\!\! compression \!&\!\! quantization \!&\!\! error feedback\! & \!momentum\! \\
    \hline
     Distributed SGD& $\surd$ & $\times$ & $\times$ & $\times$ & $\times$& $\times$\\
     \hline
     Alistarh et al. \cite{alistarh2017qsgd} & $\surd$ &$\times$ & $\times$ &$\surd$& $\times$& $\times$\\
     \hline
     \!\! Karimireddy et al. \cite{karimireddy2019error}\! &$\surd$& $\times$ &one side& $\times$ & $\surd$&$\times$\\
     \hline
     Hou et al. \cite{hou2018analysis}& $\times$&$\surd$& $\times$&$\surd$&$\times$&$\times$\\
     \hline
     Zheng et al. \cite{zheng2019communication} &$\surd$& $\times$& $\surd$& $\times$&$\surd$&$\surd$\\
     \hline
     \secchange{Chen et al. \cite{chen2021quantized}} &\secchange{$\surd$}& \secchange{$\surd$}& \secchange{one side + weights}& \secchange{$\times$}&\secchange{$\surd$}&\secchange{$\surd$}\\
     \hline
     Efficient-Adam &$\surd$&$\surd$&$\surd$&$\surd$&$\surd$&$\surd$\\
    \hline
    \end{tabular}
    \label{comparison}
%    \vspace{-0.5cm}
\end{table*}

Optimizing a large-scale stochastic non-convex function has been studied for many years, in which distributed SGD method has been widely explored. For distributed SGD, its convergence rate has been established \change{both in} the Parameter-Server model \cite{agarwal2011distributed} and decentralized model \cite{lian2017can}. To reduce the communication cost, compression on the communication has been added into the distributed stochastic methods, such as QSGD \cite{alistarh2017qsgd} and Sign SGD \cite{bernstein2018signsgd}. However, due to the error introduced by compression, most distributed methods will fail to converge with compressed communication. Several works \cite{jiang2018linear,wangni2018gradient,wen2017terngrad} have tried to use unbiased compressors to remove errors brought by compression and provide related convergence analyses. However, with some biased compressors (e.g. SignSGD), SGD methods empirically achieve good \change{solutions}, which cannot be explained by those works. To deal with biased compression, Karimireddy et al. \cite{karimireddy2019error} first introduced error feedback into SignSGD method and established its convergence. However, their analysis only focuses on the setting in a single machine. In addition, there exist several decentralized SGDs that try to utilize error feedback to reduce compression errors, such as ChocoSGD \cite{koloskova2019decentralized} and DeepSqueeze \cite{tang2019texttt}. 

Further, for the parameter-server model, Zheng et al. \cite{zheng2019communication} introduced a two-way error feedback technique into SGD with momentum and showed its convergence in the nonconvex setting. Moreover, they used a block-wise compressor to meet the compressor assumption. However, even with the block-wise compression, their assumption still does not hold when the gradient goes to $0$. Besides, the learning rate for each worker may be hard to tune when the number of workers is large. On the other hand, Hou et al. \cite{hou2018analysis} adapted the distributed Adam with quantized gradients and weights to reduce the communication overhead and gave the convergence bound for the convex case. However, in their work, they considered an unbiased quantization function and did not consider error feedback terms, which may limit the use of the analysis. Besides, Chen et al. \cite{chen2021quantized} proposed Quantized Adam, but they only consider error feedback in one direction, and in the other direction, weight quantization is used. Wang et al. \cite{wang2022communication} compressed and communicated with gradients, then they performed an AMSGrad algorithm.  \change{The update equation can be encapsulated as $x_{t+1} = x_t - \alpha_t m_t /\sqrt{v_t + \epsilon}$. The theoretical analysis conducted by Chen et al. \cite{chen2021quantized} and Wang et al. \cite{wang2022communication} significantly depends on the constant $\epsilon$, which must be sufficiently large. However, a larger value of $\epsilon$ can lead to it dominating the denominator of the equation, thereby causing the algorithm to regress to a simple stochastic gradient descent mechanism. Further, Doostmohammadian M et al. \cite{doostmohammadian2022distributed, doostmohammadian2022fast} and Magnusson et al. \cite{magnusson2018communication} present algorithms that solve constrained optimization problems when the communication network experiences heterogeneous time-delays.}  Moreover, when considering transmitting bits among devices, the compression assumption used in previous work does not hold (e.g. with finite bits arbitrary small \change{values} can not be represented). Therefore, we will give a practical assumption on the compression of the communication. To distinguish from the previous assumption, we denote the function satisfying the new assumption as a quantizer, and the function satisfying the previous assumption as a compressor.
\secchange{Among the research we've looked at, the work by Chen et al. \cite{chen2021quantized} is the most similar to ours. However, they tried to reduce how much information the server sends to the workers by using shorter representations for the model's weights which further introduces additional quantization bias. Since it's tough to represent these weights accurately with just a few bits, we took a different path in our study. We try to send compressed updates for the weights instead. This way, we're able to lower the communication needed even more. }

Different from these existing works, we propose an adaptive distributed stochastic algorithm with a two-way quantization/compressor and a two-way error feedback strategy, dubbed Efficient Adam. We also establish its iteration complexity and communication complexity in terms of bits under some specialized designed quantization mappings in the non-convex setting.  In the following Table \ref{comparison}, we summarize the differences between our proposed Efficient-Adam and several existing communication-efficient distributed SGD methods in the parameter-server model.

%There exist several works on distributed SGD \cite{???,???} by utilizing the error-feed back \cite{XXXX} ...  

%The most related work \cite{hou2018analysis,zheng2019communication}. 

%Hou et al. \cite{hou2018analysis} XXXX

%Zheng et al \cite{zheng2019communication}...
%\vspace{-0.2cm}
\section{Preliminaries}
\label{preliminary-sec}
For presenting and analyzing the proposed Efficient-Adam in the following sections, we first give several basic definitions and assumptions. First, we denote the stochastic estimation of $F(x) = \mathbb{E}_\xi[f(x,\xi)]$ as $g=\nabla_x f(x,\xi)$ with a given sample $\xi$. In addition, we summarize a few necessary assumptions on function $F$ and stochastic estimate $g$, and define the $\varepsilon$-stationary point of problem \eqref{problem}. 
%\vspace{-0.2cm}
\begin{assumption}\label{assumption-f-grad}
Function $F(x)$ is assumed to be lower bounded, i.e., $F(x) \ge F^*>-\infty$ for some give constant $F^*$.  Gradient $\nabla F$ is assumed to be L-Lipschitz continuous, i.e., $\|\nabla F(x) - \nabla F(y)\|\leq L\|x-y\|$.  Besides, we assume that for given $x$, $g$ is an unbiased estimator of the gradient with bounded second moment, i.e., $\mathbb{E}_\xi(g) = \nabla{F}(x) = \nabla_x \mathbb{E}_\xi(f(x,\xi))$ and $\| g \|\leq G$. 
\end{assumption}
%\vspace{-0.3cm}
\begin{definition}\label{stationary-point-def}
We denote $x^* \in \mathbb{X}$ as the  $\varepsilon$-stationary point of problem \eqref{problem}, if it satisfies $\mathbb{E}\|\nabla F(x^*)\|^2 \leq \varepsilon$.
\end{definition}
%\vspace{-0.2cm}
Assumption \ref{stationary-point-def} is widely used for establishing convergence rates of adaptive SGD methods towards $\varepsilon$-stationary points, such as Adam/RMSProp \cite{zou2019sufficient}, AdaGrad \cite{ward2018adagrad,zou2018weighted,defossez2020convergence}, AMSGrad \cite{reddi2019convergence,chen2018convergence} and AdaSAM \cite{sun2023adasam}. \change{Recently, Shi et al. \cite{shi2021rmsprop} and Zhang et al. \cite{zhang2022adam} have eliminated the bounded gradient assumption in their studies. However, their results are only applicable when $\mathbb{P}$ has finite supports, and these cannot be extended to the general stochastic setting.}
%Such as Zou et al. \cite{zou2019sufficient,zou2018weighted} relax $\| g \|\leq G$ to $\mathbb{E}_{\xi}[\| g \|] \leq G$ in Assumption \ref{assumption-f-grad} to establish convergence rate of Adam and AdaGrad, Reddi et al. \cite{reddi2019convergence} and Chen et al. \cite{chen2018convergence} use Assumption \ref{assumption-f-grad} to establish convergence rate of AMSGrad. Ward et al. \cite{reddi2019convergence} and Defossez et al. \cite{defossez2020convergence} use Assumption \ref{assumption-f-grad} to establish convergence rate of AdaGrad. 

Below, we define a class of quantization mappings which are used to quantize the communicated information between the workers and server in Figure \ref{parameter-server} to reduce the communication. 
%\vspace{-0.2cm}
\begin{definition}\label{quantization-mapping-def}
We call $\mathcal{Q}: \mathbb{X}\to \mathbb{X}$ a quantization mapping, if there exist constants $\delta > 0$ and $\delta'\ge 0$ such that the two inequalities hold: $\|x-\mathcal{Q}(x)\|\leq (1-\delta)\|x\| + \delta'$ and $\|\mathcal{Q}(x)\|\leq (2-\delta)\|x\|$. 
\end{definition}
%\vspace{-0.2cm}
We give a few quantization mappings that satisfy the above definition. %In addition, we give the bit analysis that the quantization function needs in one communication round.
%\vspace{-0.15cm}
\begin{example}\label{quantization-mapping1}
Let $M_1$ be the set of 32-bits floating point number and $M_2 = \{\frac{i}{2^k-1}| i = -2^k+1, -2^k+2, \cdots, 2^k -1\}$, for some positive integer k. Then define $Q_1(x) = \arg\min_{y\in M_1} |x-y|$ which maps $\mathbb{R}$ to $M_1$, and $Q_2(x) = \arg\min_{y \in M_2^d} \|x-y\|_2$ which maps from $\mathbb{R}^d$ to $M_2^d$. The quantization function $\mathcal{Q}$ is defined as $\mathcal{Q}(x) = Q_1(\|x\|_\infty) Q_2(x/Q_1(\|x\|_\infty))$, when $Q_1(\|x\|_\infty)>0$ and $\mathcal{Q}(x) = 0$, when $Q_1(\|x\|_\infty) = 0$.
\end{example}
%\vspace{-0.3cm}
\begin{example}\label{quantization-mapping2}
Let $M = \{2^{k}, 2^{k+1}, \cdots, 2^{K}\}$ for some integer numbers $k$ and $K>k$. Define the quantization function $\mathcal{Q}$ as $\mathcal{Q}(x) = sign(x) \arg\min_{y \in M^d} \||x|-y\|$. 
\end{example}

There exist several works that utilize compressor mappings \cite{stich2018sparsified,stich2019error} to reduce the communication cost for distributed SGD in the Parameter-Server model, in which compressor mapping is formally defined as follows: 
%\vspace{-0.15cm}
\begin{definition}
\label{compressor-def}
We call $\mathcal{Q}:\mathbb{X}\to\mathbb{X}$ as compressor mapping \cite{stich2018sparsified,stich2019error,zheng2019communication}, if there exists a positive value $\delta$ such that the inequality holds: $\|x-\mathcal{Q}(x)\|\leq (1-\delta)\|x\|$. Based on the definition, it directly holds that $\|\mathcal{Q}(x)\|\leq (2-\delta) \|x\|$.
\end{definition}
% \vspace{-0.35cm}
\begin{remark}\label{quantization-compressor}
Most quantization mappings do not belong to the above class of compressors, since, with a finite set as range, quantization mapping cannot achieve the condition of compressors for arbitrarily small \change{parameters} $x$. Therefore, we cannot find any quantization function that satisfies the definition of the compressor in Definition \ref{compressor-def}. Moreover, when $\delta'$ is much smaller than the precision which we want to achieve, we can ignore it and reduce the condition to the compressor condition. Meanwhile, we have an additional condition $\|Q(x)\|\leq (2-\delta) \|x\|$, which can be easily achieved by replacing rounding to flooring. Therefore, we give a more practical condition for analyzing the communication complexity. Moreover, the convergence of Efficient-Adam with quantization can be easily extended to compressors. 
\end{remark}

%\vspace{-0.5cm}
\section{Efficient-Adam}\label{efficient-adam-sec}
%\vspace{-0.1cm}
In this section, we describe the proposed Efficient-Adam, whose workflow has already been displayed in Figure \ref{parameter-server}. To make the presentation clear, we split Efficient-Adam into two parts: parameter server part (Algorithm \ref{alg1}) and $N$-workers part (Algorithm \ref{alg2}). To reduce the communication overhead among the server and workers, we introduce a two-way quantization mapping. We denote the quantization function in the parameter server as $\mathcal{Q}_s(\cdot)$ and the quantization function in the workers as $\mathcal{Q}_w(\cdot)$, respectively. The detailed iterations of Efficient-Adam are described in Algorithm \ref{alg1} and Algorithm \ref{alg2}. 

In the parameter server (see Algorithm \ref{alg1}), the initial value of $x$ will be broadcast to each worker at the initialization phase. Then in each iteration, the server gathers the update from each worker, calculates the average of these updates, and then broadcasts this average after a quantization function.

\begin{algorithm}[htb!]
\caption{\ Efficient-Adam: the server}
\label{alg1}
\begin{algorithmic}[1]
\STATE {\bf Parameters:} Choose quantization mapping $\mathcal{Q}_s(\cdot)$ 
via Definition \ref{quantization-mapping-def}. {Initialize $x_1 \in \mathbb{X}$ and error term $e_1=0$}.
\STATE Broadcasting $x_1$ to all workers;
\FOR{$t =1,2,\cdots,T$}
\STATE Gathering and averaging all updates from workers $\hat{\delta}_t = \frac{1}{N}\sum_{i = 1}^N \delta_t^{(i)}$;
\STATE Broadcasting $\tilde{\delta}_t = \mathcal{Q}_s(\hat{\delta}_t + e_t)$;
\STATE $e_{t+1} = \hat{\delta}_t + e_t - \tilde{\delta}_t$;
\STATE $x_{t+1} = x_{t}-\tilde{\delta}_t$;
\ENDFOR
\end{algorithmic}
{\bf Output:} $x_{T+1}$.
\end{algorithm}

\begin{algorithm}[htb!]
\caption{\ Efficient-Adam: the $i$-th worker}
\label{alg2}
\begin{algorithmic}[1]
\STATE {\bf Parameters:} Choose hyperparameters $\{\alpha_t\},\beta,\{\theta_t\}$, and quantization mapping $\mathcal{Q}_w(\cdot)$ satisfying Definition \ref{quantization-mapping-def}. Set initial values $m^{(i)}_0 = 0$ , $v^{(i)}_0 = \epsilon$, and error term $e^{(i)}_1 = 0$, respectively.
\STATE Receiving $x_1$ from the server;
\FOR{$t =1,2,\cdots,T$}
\STATE Sample a stochastic gradient of $f(x_t)$ as $g^{(i)}_t$;
\STATE $v^{(i)}_t = \theta_t v^{(i)}_{t-1} + (1-\theta_t) \big[g^{(i)}_t\big]^2$;
\STATE $m^{(i)}_t = \beta m^{(i)}_{t-1} + (1-\beta)g^{(i)}_t$;
\STATE Sending $\delta_t^{(i)} = \mathcal{Q}_w\left(\alpha_t{m^{(i)}_t}/{\sqrt{v^{(i)}_t}}+e^{(i)}_t\right)$;
\STATE $e^{(i)}_{t+1} = \alpha_t\frac{m^{(i)}_t}{\sqrt{v^{(i)}_t}}+e^{(i)}_t- \delta^{(i)}_t$;
\STATE Receiving $\tilde{\delta}_t$ from the server;
\STATE $x_{t+1} = x_t - \tilde{\delta}_t$;
\ENDFOR
\end{algorithmic}
\end{algorithm}

In each worker (see Algorithm \ref{alg2}), at the initial phase, the initial value of $x$ will be received. In each update iteration, a worker will sample a stochastic gradient $g_t$ and calculate the update vector $\delta_t$. Then it will send the update vector to the server with a quantization mapping and receive the average update vector from the server. Finally, the worker will update its local parameters. 

For the above two algorithms, we use the error-feedback technique to reduce the influence of errors introduced by quantization shown in line 6 of Algorithm \ref{alg1} and line 8 in Algorithm \ref{alg2}. \change{The error-feedback} technique can be viewed as delaying updating values for $x$, while without error-feedback the algorithm discards \change{a} few updating values on $x$ which losses information during the optimization process. Therefore, the error-feedback technique can help the algorithm by preserving information discarded by the quantizer.

%\vspace{-0.3cm}
\subsection{Complexity Analysis}

 In this subsection, we give iteration complexity analysis and bit communication complexity analysis in order to obtain an $\varepsilon$-stationary solution of problem \eqref{problem} in Definition \ref{stationary-point-def}, i.e., $$\mathbb{E}\|\nabla F(x)\|^2 \leq \varepsilon.$$
 We denote $\delta_s$ and $\delta_s'$ as the constants defined in Definition \ref{quantization-mapping-def} for $\mathcal{Q}_s(\cdot)$, and $\delta_w$, $\delta_w'$ for $\mathcal{Q}_w(\cdot)$. In addition, we further make another assumption on hyperparameters $\theta_t$ and $\alpha_t$.
%\vspace{-0.2cm}
 \begin{assumption}\label{hyper-parameters}
 For a given maximum number of iterations $T$, we define the exponential moving average parameter $\theta_t = 1-\frac{\theta}{T}$, and base learning rate $\alpha_t = \frac{\alpha}{\sqrt{T}}$. Besides, we define $\gamma = \beta/(1-\frac{\theta}{T})$. 
 \end{assumption} 
%\vspace{-0.2cm}
Below, we characterize the iteration complexity of Efficient-Adam to achieve an $\varepsilon$-stationary solution of problem \eqref{problem}. In the corollary, we characterize the overall bit communication complexity of Efficient-Adam with the quantization mapping defined in Example \ref{quantization-mapping2}. 
%\vspace{-0.2cm}
\begin{theorem}\label{complexity-quantization}
Let $\{x_t\}$ be the point generated by Algorithm \ref{alg1} and Algorithm \ref{alg2}. In addition, assume that all workers work identically and independently. Let $x_\tau^T$ be the random variable $x_\tau$ with $\tau$ taking from $\{1,2,\cdots, T\}$ with the same probability. For given $\varepsilon$, when {\small $T\geq \frac{4(G^2+\epsilon d)}{(1-\beta)^2\alpha^2 \varepsilon^2}\left(F(x_1)-F^* + C_{2} \right)^2$}, it always holds that
\begin{gather*}
 \mathbb{E}[\|\nabla F(x_\tau^T)\|^2] 
 \leq \varepsilon  \\ 
  \!+\!\left(\frac{(2\!-\!\delta_s)\delta_s'}{\delta_s}+\frac{(2\!-\!\delta_s)(2\!-\!\delta_w)\delta_w'}{\delta_s\delta_w}\right)\frac{2L\sqrt{G^2\!+\!\epsilon d}C_3}{(1-\beta)}, 
 \end{gather*}
 where
{$C_1 = \left(\frac{\beta/(1-\beta)}{\sqrt{(1-\gamma)\theta_1}}+1\right)^2$, $C_{2} \!=\! \frac{d}{1-\sqrt{\gamma}}\left(\frac{(2-\delta_w)(2-\delta_s)\alpha^2L}{\theta(1-\sqrt{\gamma})^2\delta_w\delta_s} \!+\! \frac{2C_1G\alpha}{\sqrt{\theta}} \right) \left[\log\left(1+\frac{G^2}{\epsilon d}\right)\!+\!\frac{\theta}{1\!-\!\theta}\right]$, and $C_3 =\sqrt{\frac{d}{\theta(1-\sqrt{\gamma})^4}\left[\log\left(1+\frac{G^2}{\epsilon d}\right)+\frac{\theta}{1-\theta}\right]}.$}
 \end{theorem}

\begin{remark}
    \change{In comparison to the work of Chen et al. \cite{chen2021quantized} and Wang et al. \cite{wang2022communication}, we have an order of $\mathcal{O}(\log^2(1+\frac{1}{\epsilon d}))$ with respect to the constant $\epsilon$, which pertains to numerical issues. Conversely, they demonstrate an order of $\mathcal{O}(\frac{1}{\epsilon})$. In practical applications, $\epsilon$ is typically configured to be a small value. Therefore, our algorithm is capable of achieving significantly faster convergence compared to theirs. }\secchange{ 
Additionally, because Chen et al. \cite{chen2021quantized} compress the weights instead of the updates, when we let the value of $T$ go to infinity, their method can't reach the exact stationary point. Instead, it gets really close to some stationary point, even when we use the compressor in Definition \ref{compressor-def}. On the other hand, our approach can reach the stationary point with $T$ going to infinity and compressors in Definition \ref{compressor-def}.} 
\end{remark}
% \vspace{-0.3cm}
 \begin{corollary}\label{bit-complexity-quantization}
 With the quantization function in Example \ref{quantization-mapping2}, when we want to get an $\varepsilon$-stationary solution we need at most $C_5$ bits per iteration on the workers and the server. Besides, to achieve an $\varepsilon$-stationary point, we need $\frac{16(G^2+\epsilon d)}{(1-\beta)^2\alpha^2 \varepsilon^2}\left(F(x_1)-F^* + C_{4} \right)^2$ iterations, where
{$C_{4} = \frac{d}{1-\sqrt{\gamma}}\left(\frac{9\alpha^2L}{\theta(1-\sqrt{\gamma})^2} + \frac{2C_1G\alpha}{\sqrt{\theta}} \right) \left[\log\left(1\!+\!\frac{G^2}{\epsilon d}\right)+\frac{\theta}{1-\theta}\right]$, $C_5 = d\left(1\!+\!\log\left(\log(\frac{\alpha}{(1\!-\!\gamma)\theta})\!+\!\log\left(\frac{24Ld\sqrt{G^2\!+\!\epsilon d}C_3}{\varepsilon(1-\beta)}\right)\!+\!1\right)\right)$}, and {{$C_1,\ C_3$}} are defined in Theorem \ref{complexity-quantization}.
 \end{corollary}
 %\begin{proof}
 %By solving $\left(\frac{\delta_s'}{\delta_s} +\frac{\delta_w'}{\delta_s\delta_w}\right)\frac{2L\sqrt{G^2+\epsilon d}C_8}{(1-\beta)}\leq \frac{\varepsilon}{2}$, and set $K$ to be the same in both direction, we obtain that when $K \geq \log\left(\frac{8Ld^2\sqrt{G^2+\epsilon d}C_8}{\varepsilon(1-\beta)(1-2^{-m})^2}\right)-m + 1$, this equality always holds. Besides, using theorem \ref{complexity-quantization}, with inequality $\mathbb{E}[\|\nabla f(x_\tau^T)\|^2 \leq \varepsilon/2 +  \left(\frac{\delta_s'}{\delta_s} +\frac{\delta_w'}{\delta_s\delta_w}\right)\frac{2L\sqrt{G^2+\epsilon d}C_8}{(1-\beta)}$, we have $T\geq \frac{16(G^2+\epsilon d)}{(1-\beta)^2\alpha^2 \varepsilon^2}\left(f(x_1)-f^* + C_{11} \right)^2$. Moreover, in each iteration we need to transmit $d(m + \log(\log(G)+K+1))$ bits.
 %By taking $m=1$, we get the desired result.
 %\end{proof}
%\vspace{-0.15cm}
From the above theorem and corollary,  we can reduce the bits of communication by two-way quantization and error feedback. In addition, the bit communication complexity and iteration complexity are $O(d\log\log \frac{d^2}{\varepsilon})$ and $O(\frac{d^3}{\varepsilon^2})$, respectively. 
Moreover, it can be seen that adding compression in both sides can converge to an arbitrary $\varepsilon$-stationary point when we have large enough communication bandwidth and sufficient iterations. There is a limited influence when we add quantization into communication if we can have a small additional error $\delta_w'$.  Below, we show that if we replace the quantization mapping in Definition \ref{quantization-mapping-def} as Compressor in Definition \ref{compressor-def} in Algorithms \ref{alg1}-\ref{alg2}, Efficient-Adam can attain the same order of iteration complexity as original Adam without introducing additional error terms.

%\vspace{-0.15cm}
\begin{corollary}
\label{compressor-complexity}
% Let $\{x_{t}\}$ be the point generated by Algorithms \ref{alg1}-\ref{alg2} with quantization mapping replaced by compressor as Definition \ref{compressor-def}. Assume that all workers works identically and independently.  In addition, let $x^T_\tau$ be the random variable $x_\tau$ with $\tau$ taking from $\{1,2,\ldots,T\}$ with the same probability. the convergence result holds as follows:  
 %\[
%\mathbb{E}[\|\nabla f(x_\tau^T)\|^2]\leq \frac{1}{\sqrt{T}} \frac{2\sqrt{G^2+\epsilon d}}{\alpha(1-\beta)}(f(x_1) - f^* + C_2),
%\]
Let $\{x_t\}$ be the point generated by Algorithm \ref{alg1} and Algorithm \ref{alg2} with the quantization mapping replaced by the compressor in Definition \ref{compressor-def}. In addition, assume that all workers work identically and independently. Let $x_\tau^T$ be the random variable $x_\tau$ with $\tau$ taking from $\{1,2,\cdots, T\}$ with the same probability. For given $\varepsilon$, when {$T\geq \frac{4(G^2+\epsilon d)}{(1-\beta)^2\alpha^2 \varepsilon^2}\left(F(x_1)-F^* + C_{2} \right)^2$}, it always holds: 
{$
\mathbb{E}[\|\nabla F(x_\tau^T)\|^2] \leq \varepsilon,
$}
where {$C_1 \!=\! \Big(\frac{\beta/(1-\beta)}{\sqrt{(1-\gamma)\theta_1}}+1\Big)^2$}, 
and { $C_{2} \!=\! \frac{d}{1-\sqrt{\gamma}}\left(\frac{(2-\delta_w)(2-\delta_s)\alpha^2L}{\theta(1-\sqrt{\gamma})^2\delta_w\delta_s} \!+\! \frac{2C_1G\alpha}{\sqrt{\theta}} \right) \left[\log\left(1+\frac{G^2}{\epsilon d}\right)\!+\!\frac{\theta}{1\!-\!\theta}\right]$}.
\end{corollary}

From the above corollary, it can be shown that the convergence rate of Efficient-Adam matches the convergence rate of the vanilla Adam. Thanks to the two-way error feedback techniques employed in Efficient-Adam, the convergence rate will not be hurt by the introduced bi-direction compressor as defined in Definition \ref{compressor-def}, which merely affects the speed of convergence at the constant level.
\secchange{However, Chen et al. \cite{chen2021quantized} adopt the compression on the weights which merely converges to the neighborhood of the stationary point.}
%\subsection{Proof Sketch of Theorem \ref{complexity-quantization}}
%Due to the limited space, we will only give the proof sketch of Theorem \ref{complexity-quantization}. The full proof can be found in the Supplementary Materials. 

\begin{table*}[htb!]
%\vspace{-0.2cm}
    \centering
    \caption{Abbreviation of all compared algorithms}
%    \vspace{-0.2cm}
    \renewcommand{\arraystretch}{1.2}
    \begin{tabular}{|c|c|c|c|}
    \hline
     \!Abbreviation\! & \!optimization algorithm\!&\!\!\! communication algorithm \!\!\! &\!\! quantization function \!\! \\
    \hline
     {\bf Ours\_full}       & Efficient-Adam &  None           & 32-bits floating-point number\\
     {\bf Ours\_com1}       & Efficient-Adam & Error-feedback  & Example \ref{quantization-mapping1}\\
     {\bf Ours\_com1\_err}  & Efficient-Adam & None            & Example \ref{quantization-mapping1}\\
     {\bf Ours\_com2}       & Efficient-Adam & Error-feedback  & Example \ref{quantization-mapping2}\\
     {\bf Ours\_com2\_err}  & Efficient-Adam & None            & Example \ref{quantization-mapping2}\\
     {\bf Sgdm\_full}       & Distributed SGD with Momentum & None & 32-bits floating-point number \\
     {\bf Zheng et al. }       & Distributed SGD with Momentum & Error-feedback & Example \ref{quantization-mapping1} \\
     {\bf Sgdm\_terngrad}       & Distributed SGD with Momentum & Terngrad & Example \ref{quantization-mapping1} \\
     {\bf Dadam\_full} & Distributed Adam & None & 32-bits floating-point number\\
     {\bf Dadam\_terngrad} & Distributed Adam & Terngrad & Example \ref{quantization-mapping1} \\
     \change{\bf Chen et al.} & \change{Distributed Adam}& \change{Error-feedback} & \change{Example \ref{quantization-mapping1}}\\
     %\change{AMSGRAD}& \change{Distributed AMSGRAD} & \change{None} & \change{32-bits floating-point number}\\
     \hline
    \end{tabular}
    \label{experimental-setting}
%    \vspace{-0.5cm}
\end{table*}

\begin{figure*}[hb!]
%\vspace{-0.6cm}
   \centering
    \!\!\!\!\includegraphics[width=0.24\textwidth]{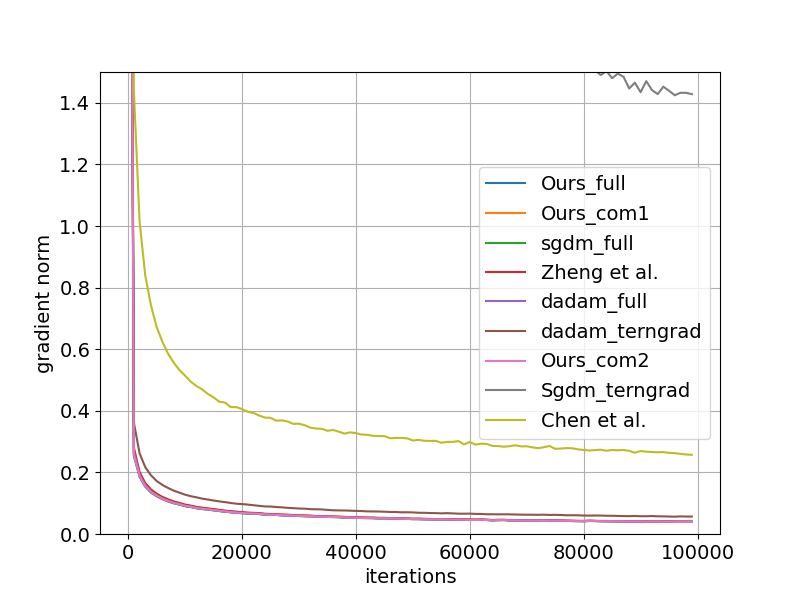}\!\!\!\!
    \!\!\!\!\includegraphics[width=0.24\textwidth]{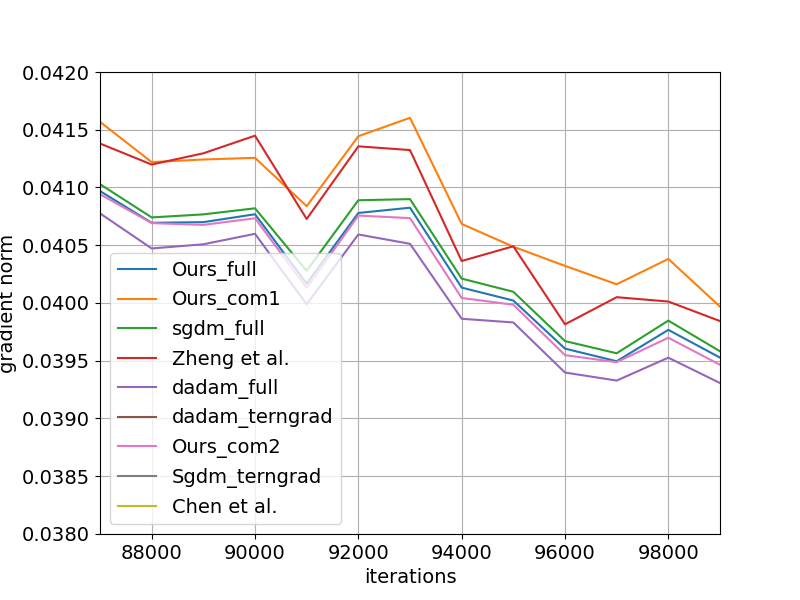}\!\!\!\!
    \!\!\!\!\includegraphics[width=0.24\textwidth]{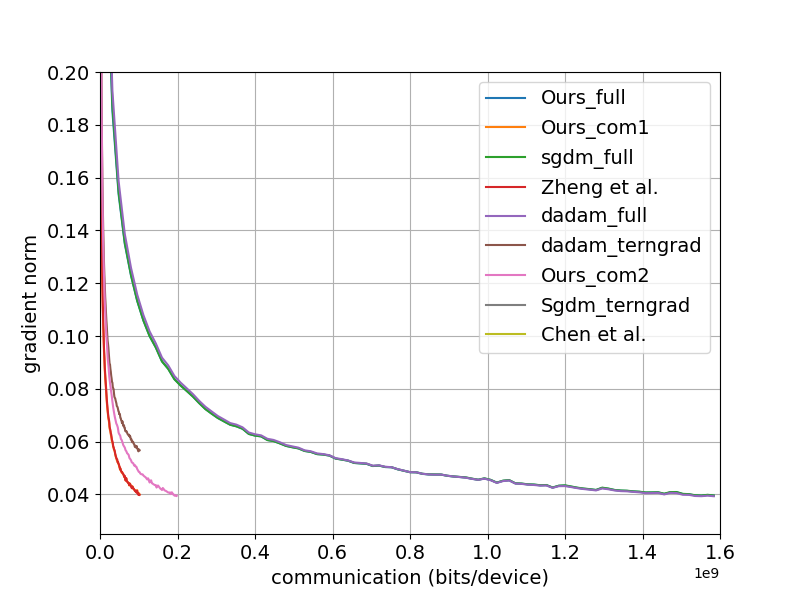}\!\!\!\!
     \!\!\!\!\includegraphics[width=0.24\textwidth]{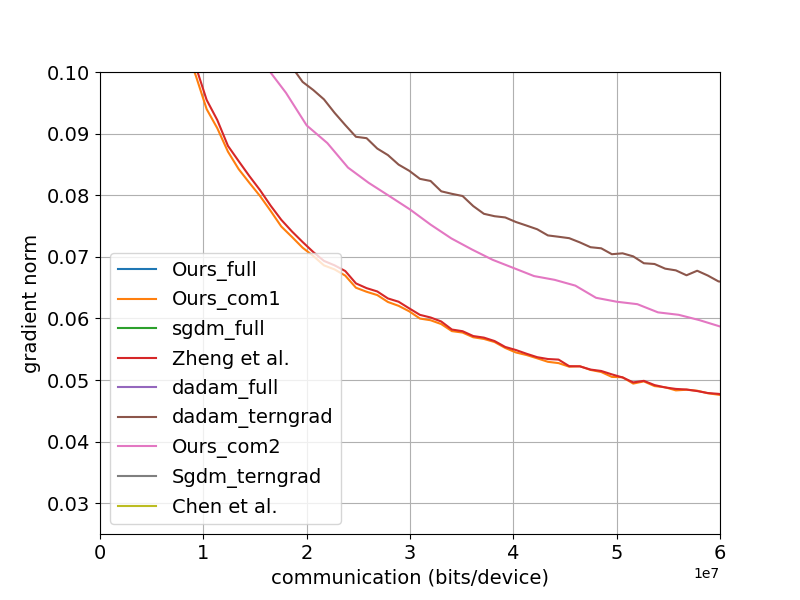}\!\!\!\!
%     \vspace{-0.3cm}
    \caption{\change{Experimental results for stochastic convex optimization \eqref{toy-example}. The second subfigure shows the zoom-in version of the first subfigure, which illustrates the norm of gradient vs. iterations on all compared algorithms. The fourth subfigure shows the zoom-in version of the third subfigure, which illustrates the norm of gradient vs. communication bit on all compared algorithms.} }
    \label{Figure1}
%    \vspace{-0.5cm}
\end{figure*}

%\vspace{-0.3cm}
\section{Experiments}\label{experiment-sec}
%\vspace{-0.1cm}
In this section, we apply our algorithms to handle several vision and language tasks. First, to accurately test the theorem on $\mathbb{E}\|\nabla f(x_t)\|^2$, we apply Efficient-Adam to tackle a simple stochastic convex case. Then we apply our algorithm for training neural networks on image classification and binary sentiment classification tasks. All of the above tasks we compare our method with including distributed SGD where the quantization part is done by Terngrad \cite{wen2017terngrad}, distributed Adam \cite{nazari2019dadam} where the quantization part is done by Terngrad, Zheng et al. \cite{zheng2019communication} \change{, and Chen et al. \cite{chen2021quantized}}. 

For simplicity, we use the abbreviation for all compared algorithms in the figures in this section. The abbreviations are summarized in the following Table \ref{experimental-setting}.

\subsection{Stochastic Convex Case}\label{convex}

In this subsection, we first optimize the following toy stochastic convex optimization problem:
\begin{align}\label{toy-example}
\min_{x} \mathbb{E}_\xi \big\|Ax - (x^*+\xi) \big\|^2,
\end{align}
where $A\in \mathbb{R}^{500\times 500}$, $x^* \in \mathbb{R}^{500}$ and $\xi$ follows a Gaussian distribution with mean 0 and covariance $0.1I$, i.e. $\xi \sim \mathcal{N}(0,0.1I)$. We randomly generate 20 cases where $A_{ij} \sim \mathcal{N}(0,1)$ independent with each other, and $x^* \sim \mathcal{N}(0,0.1I)$. For hyper-parameters, we choose $\beta = 0.9$ and $\theta = 0.99$ for both distributed Adam method and our method. For distributed SGD and Zheng et al. \cite{zheng2019communication} we choose the momentum coefficient as 0.9. As it refers to base learning rate $\alpha_{t}$, it is chosen from $\{1e-5,2e-5,\cdots, 100e-5\}$ via grid search approach. We use Example \ref{quantization-mapping1} for all algorithms as a quantizer and we use Example \ref{quantization-mapping2} for Efficient-Adam. For Example \ref{quantization-mapping1}, we set $k = 1$ for set $M_2$. Besides, for Example \ref{quantization-mapping2}, we use $k=-17$ and $K=-11$ for set $M$. 10 workers are involved in the distributed optimization process. 

The detailed comparisons are illustrated in Figure \ref{Figure1}. It shows that when full-precision algorithms achieve similar performance, the Terngrad quantizer \secchange{ and Chen et al. \cite{chen2021quantized}} will give worse results than Zheng et al. \cite{zheng2019communication} and ours. We can achieve similar results as Zheng et al. \cite{zheng2019communication}, or even better results when using Example \ref{quantization-mapping2}. Even when using Example \ref{quantization-mapping2}, it can achieve a slighter better result than the full precision version.  
In addition, as it is shown in the third subfigure and the fourth subfigure in Figure \ref{Figure1} when we consider communication bits, the quantizer introduced in Example \ref{quantization-mapping1} gives the smallest total number of bits. On the other hand,  even if Example \ref{quantization-mapping2} uses twice more bits than Example \ref{quantization-mapping1} does in one communication iteration, with our algorithm it can still converge faster than Terngrad.

\begin{figure}[H]
%    \vspace{-0.45cm}
   \centering
   \!\!\!\!\includegraphics[width=0.24\textwidth]{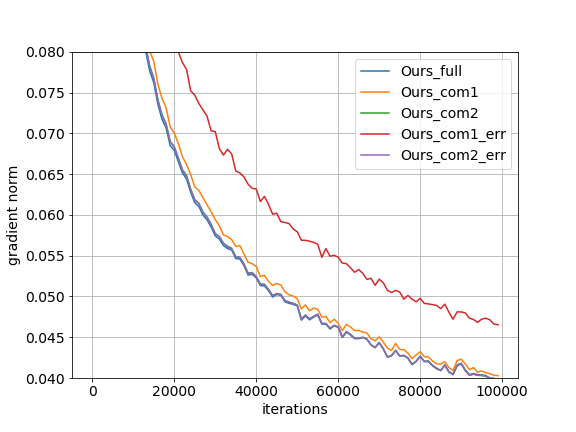}\!\!\!\!
   \!\!\!\!\includegraphics[width=0.24\textwidth]{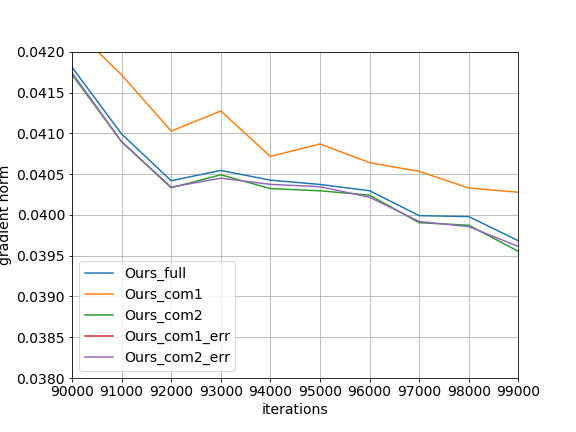}\!\!\!\!
%    \vspace{-0.5cm}
    \caption{Ablation study with/without error-feedback on the norm of gradient vs. iterations for stochastic convex optimization \eqref{toy-example}. The right figure shows the zoom-in version of the left one. 
    %``full'' represents full-precision communication. ``com1'' represents communicating with example \ref{quantization-mapping1} and ``com2'' represents communicating with Example \ref{quantization-mapping2}. 
    }
    \label{Figure2}

\end{figure}

As it refers to the error-feedback technique, shown in Figure \ref{Figure2}, with quantizer in Example \ref{quantization-mapping1}, error-feedback can help the algorithm get a better result. It helps a little when using Example \ref{quantization-mapping2} as the quantization function during communication. This may be because the Example \ref{quantization-mapping2} is accurate enough for solving this problem.

\begin{figure*}[htb!]
%\vspace{-0.3cm}
   \centering
   \!\!\!\!\includegraphics[width=0.24\textwidth]{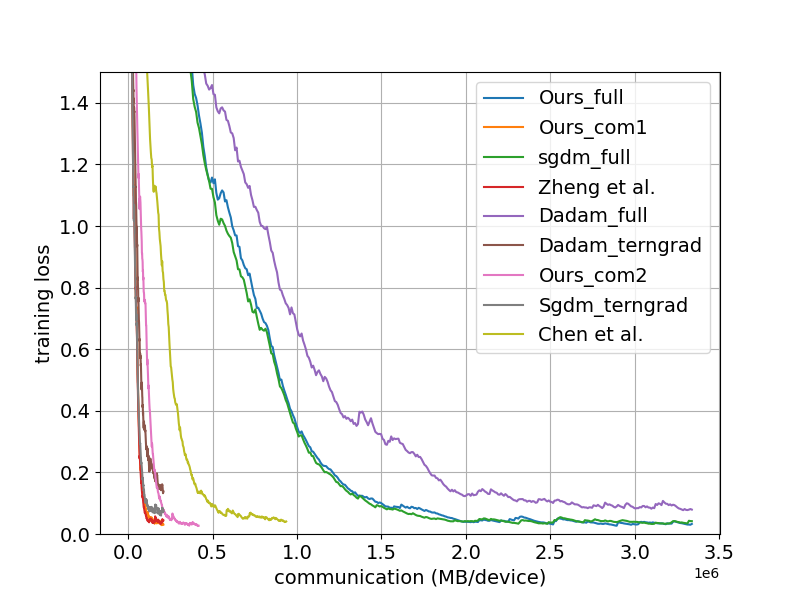}\!\!\!\!
   \!\!\!\!\includegraphics[width=0.24\textwidth]{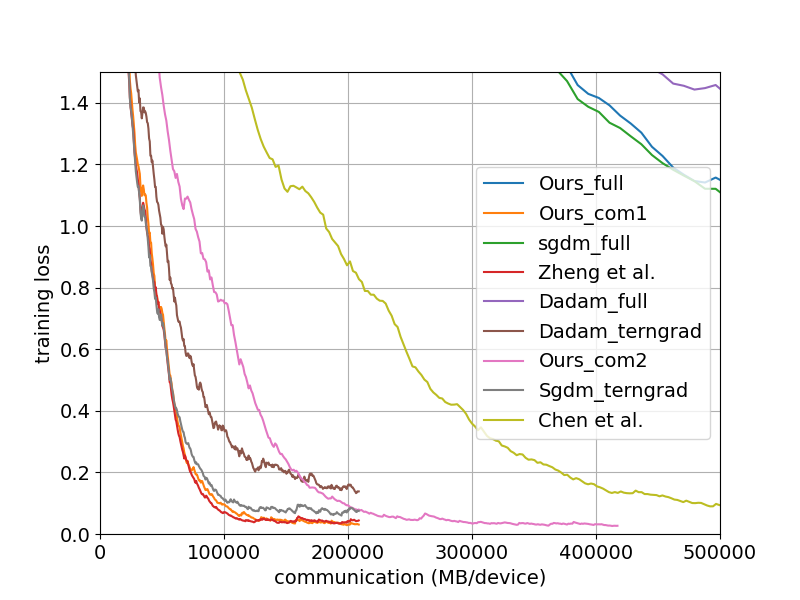}\!\!\!\!
   \!\!\!\!\includegraphics[width=0.24\textwidth]{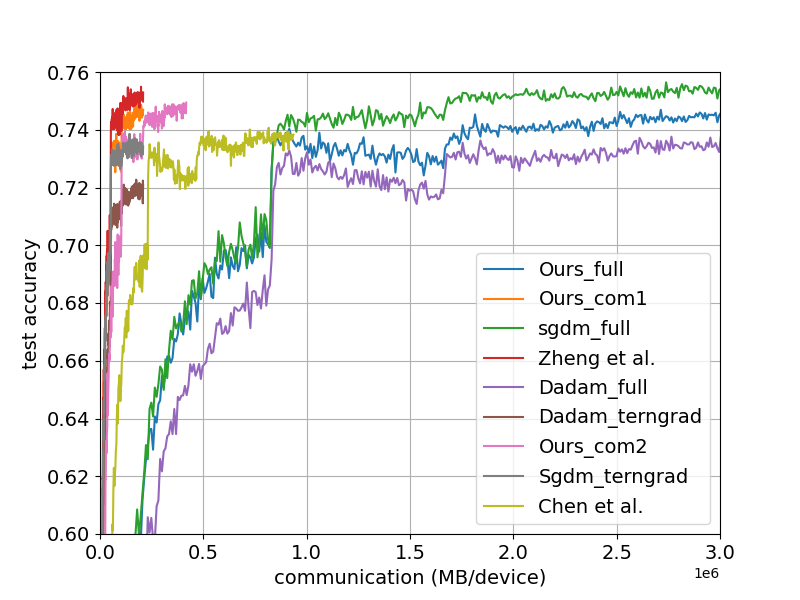}\!\!\!\!
   \!\!\!\!\includegraphics[width=0.24\textwidth]{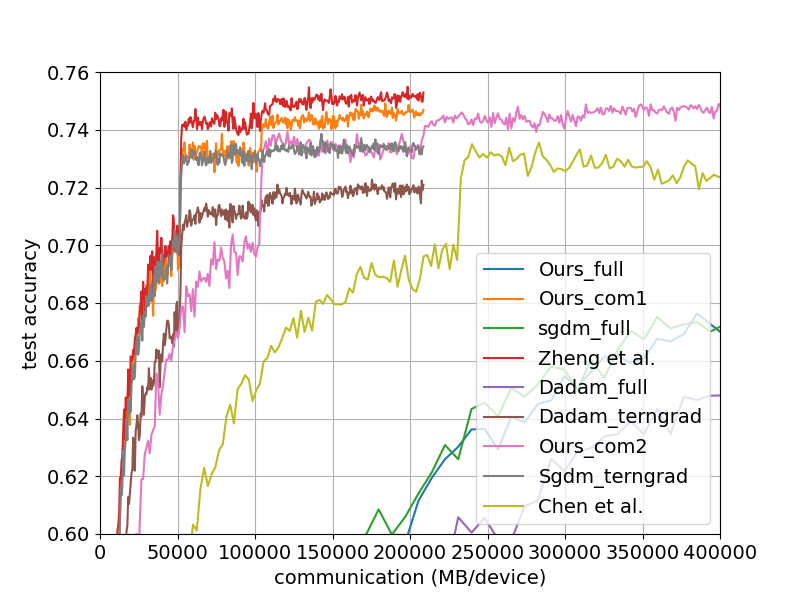}\!\!\!\!
% \vspace{-0.4cm}
    \caption{\change{Experimental results for image classification task on the CIFAR100 dataset. The left figure shows curves on training loss vs. communication bits. The second figure shows zooming into the beginning iterations of training. The third figure shows the curves on test accuracy vs. communication bits, and the fourth figure shows the zoom-in version of the third figure.} }
    \label{Fig_cifar_bits}
%    \vspace{-0.5cm}
\end{figure*}

%\vspace{-0.5cm}
\subsection{Image Classification}
\label{cifar_exp}
%\vspace{-0.1cm}

In this subsection, we apply Efficient-Adam to train a ResNet-18\cite{he2016deep} on the dataset CIFAR100 \cite{krizhevsky2009learning}.
The CIFAR100 dataset contains 60,000 $32\times 32$ images with RGB channels, and they are labeled into 100 classes. In each class, there are 600 images, of which 500 images are used for training and 100 images are used for the test.
ResNet18 contains 17 convolutional layers and one fully-connected layer. Input images will be downscaled into 1/8 size of their original size and then fed into a fully connected layer for classification training and testing. For the convolutional layers, we have 64 channels in the 1-5 convolutional layers, 128 channels in the 6-9 layers, 256 channels in the 10-13 layers, and 512 channels in the 14-17 layers. With batch size 16 and 8 workers, we train ResNet-18 for 78200 iterations.
We use cross-entropy loss with regularization to train the network. The weight of $\ell_2$ regularization we used is $5e-4$. 
For the other hyper-parameters, we use $\beta = 0.9$, $\theta_t = 0.999$ for adam-based algorithms, and for sgd with momentum we use 0.9 as the momentum coefficient. In addition, for each algorithm, based on the test accuracy, we select based learning rate denoting as $\alpha$ from $\{1e-1,5e-2,1e-2,5e-3,1e-3,5e-4,1e-4\}$ via grid search in the beginning and reduce $\alpha$ into $0.2\alpha$ after every 19,550 iterations (50 epochs for CIFAR100 training), which is often used in deep learning settings to get better performance. The hyper-parameters for the quantization function are the same as it in Section \ref{convex}. 

The detailed results are shown in Figure \ref{Fig_cifar_bits}, \ref{Fig_cifar}, and \ref{Fig_cifar_err}, where training loss has been smoothed for better presentation.  Figure \ref{Fig_cifar} shows the convergence speed and test accuracy with respect to the optimization iterations. It shows that even if we quantize the communication our algorithm and Zheng et al. \cite{zheng2019communication} can achieve similar performance as the full communication version. Meanwhile, because Terngrad \cite{wen2017terngrad} introduces extra noise, it gives the worst results. \secchange{And because Chen et al. \cite{chen2021quantized} quantized weights instead of the updates that introduce the irreducible errors, they can not approach the optimal point as close as the other methods. Thus it results in poor generalizations compared with our proposed Efficient-Adam.} When we consider the training loss and test accuracy with respect to communication bits, which is shown in Figure \ref{Fig_cifar_bits},  Zheng et al. \cite{zheng2019communication} achieves the best result. This may be because sgd-based methods are much more suitable for image classification tasks compared with adaptive stochastic gradient type methods \cite{wilson2017marginal}. Moreover, our results are similar to Zheng et al. \cite{zheng2019communication}'s results. Still, it can be shown in Figure \ref{Fig_cifar_bits} that with Example \ref{quantization-mapping2} as a quantizer, our algorithm is still better than Terngrad \cite{wen2017terngrad}.

\begin{figure}[htb!]
%\vspace{-0.4cm}
   \centering
   \!\!\!\!\includegraphics[width=0.24\textwidth]{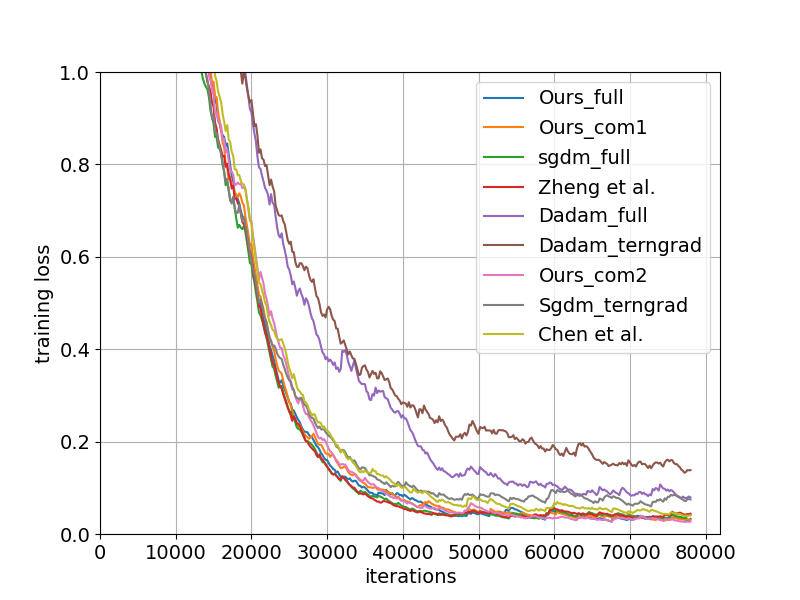}\!\!\!\!
   \!\!\!\!\includegraphics[width=0.24\textwidth]{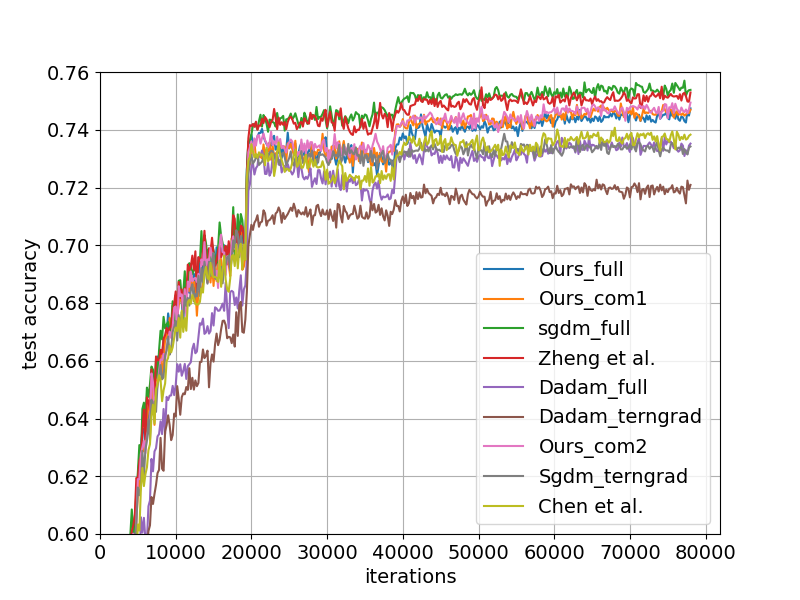}\!\!\!\!
%     \vspace{-0.3cm}
    \caption{\change{Experimental results for image classification on the CIFAR100 dataset. The left figure and right figure show curves on training loss vs. iterations and the curves on test accuracy vs. iterations, respectively. }}%``full'' represents full-precision communication. ``terngrad'' represents communicating with Terngrad \cite{wen2017terngrad}. ``com1'' represents communicating with example \ref{quantization-mapping1} and ``com2'' represents communicating with example \ref{quantization-mapping2}.}
    \label{Fig_cifar}
%    \vspace{-0.3cm}
    \end{figure}

\begin{figure}[htb!]
%\vspace{-0.5cm}
   \centering
   \!\!\!\!\includegraphics[width=0.24\textwidth]{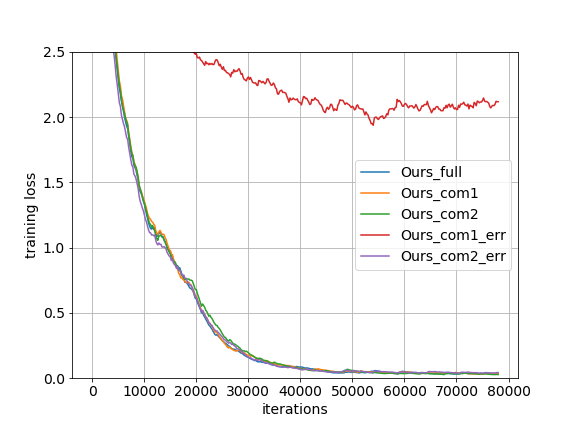}\!\!\!\!
   \!\!\!\!\includegraphics[width=0.24\textwidth]{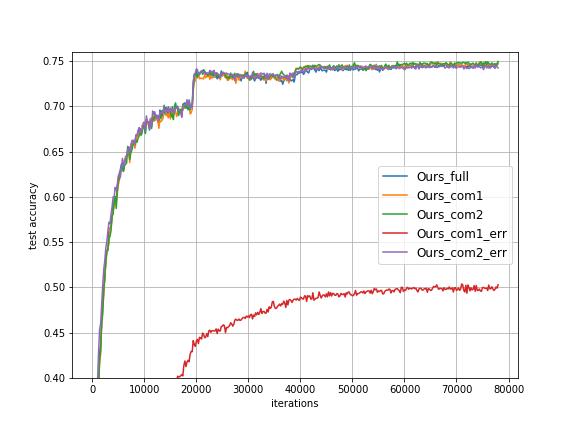}\!\!\!\!
%     \vspace{-0.2cm}
    \caption{Ablation study with/without error-feedback on the training loss and test accuracy vs. iterations for image classification task on CIFAR100 dataset, respectively. The left figure and right figure show curves on training loss vs. iterations and the curves on test accuracy vs. iterations, respectively.}% ``full'' represents full-precision communication. ``com1'' represents communicating with example \ref{quantization-mapping1} and ``com2'' represents communicating with Example \ref{quantization-mapping2}. ``err'' means running our algorithms without error-feedback technique.}
    \label{Fig_cifar_err}
%         \vspace{-0.4cm}
\end{figure}

Besides, we check whether the error-feedback technique is helpful to train the image classification task. The results are shown in Figure \ref{Fig_cifar_err}. When quantization functions introduce ``large'' error where Example \ref{quantization-mapping1} is used, the error-feedback technique helps a lot. Similarly, when the quantizer is accurate enough, where Example \ref{quantization-mapping2} is used as the quantization function, error-feedback has limited effect for improving its accuracy.

\begin{figure*}[htb!]
% \vspace{-0.2cm}
   \centering
   \!\!\!\!\includegraphics[width=0.24\textwidth]{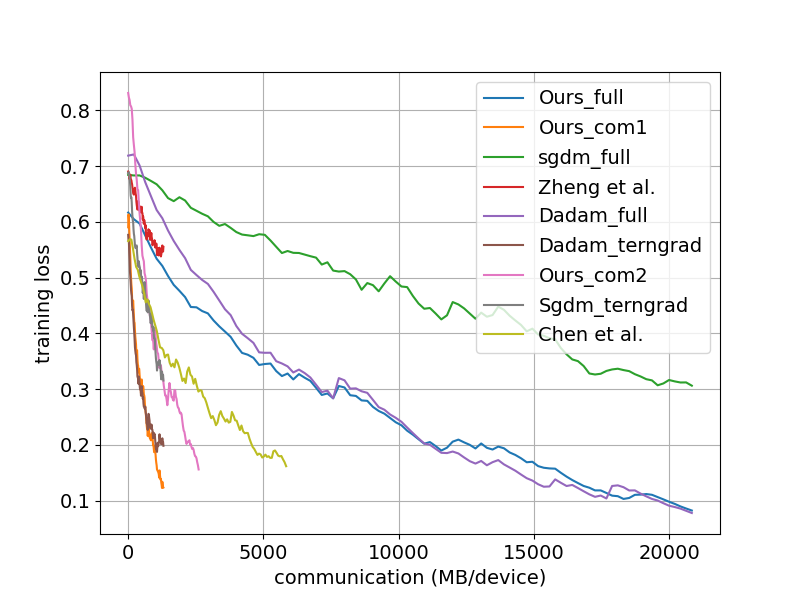}\!\!\!\!
   \!\!\!\!\includegraphics[width=0.24\textwidth]{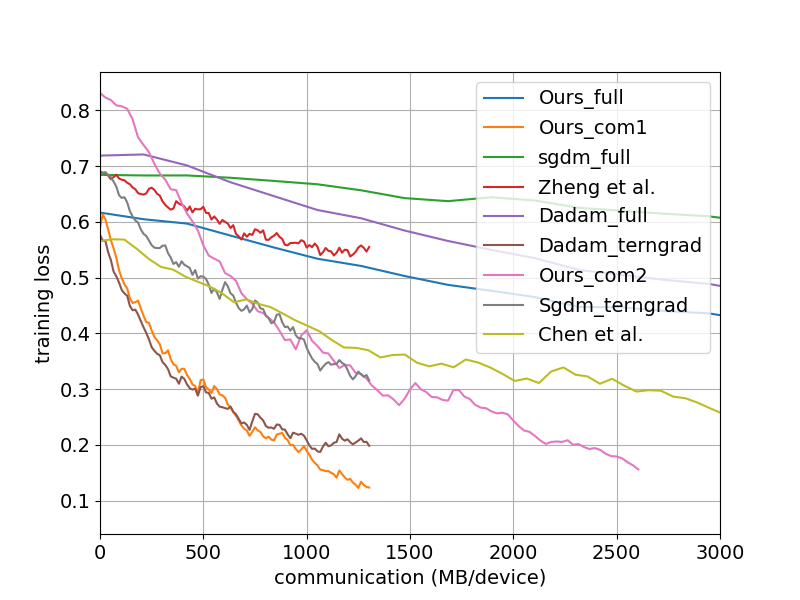}\!\!\!\!
   \!\!\!\!\includegraphics[width=0.24\textwidth]{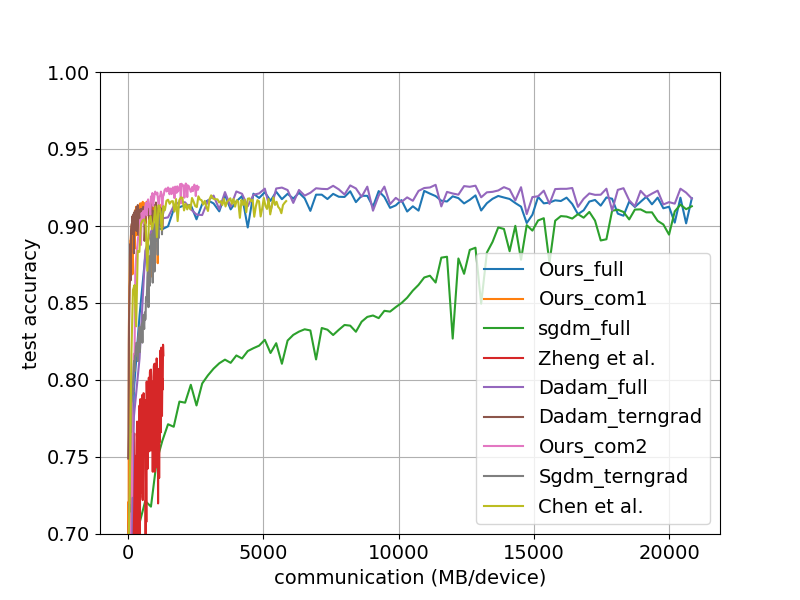}\!\!\!\!
   \!\!\!\!\includegraphics[width=0.24\textwidth]{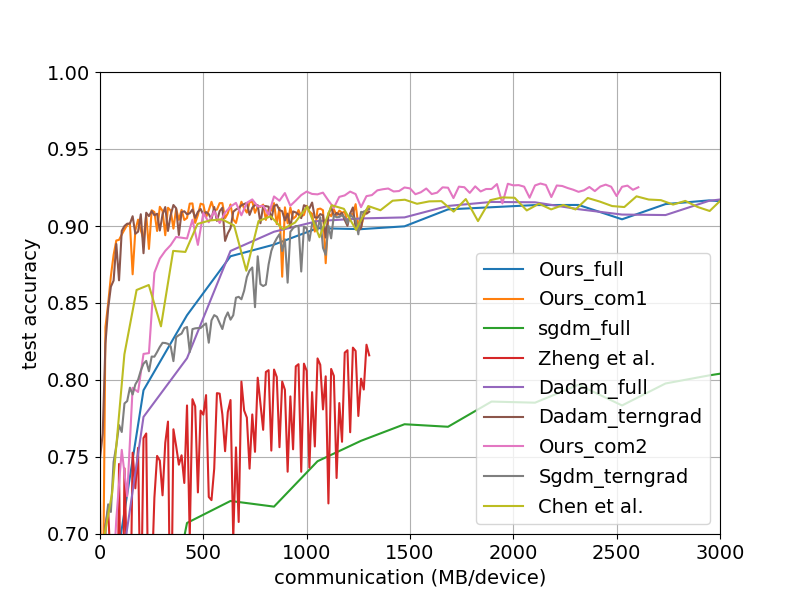}\!\!\!\!
%     \vspace{-0.35cm}
    \caption{\change{Experimental results for binary sentiment classification task on the IMDB dataset. The first subfigure shows curves on training loss vs. communication bit. The second subfigure shows zooming into the beginning iterations of training of the first subfigure. The third subfigure shows the curves on test accuracy vs. communication bits, and the fourth figure shows the zoom-in version of the third subfigure.}}% ``full'' represents full-precision communication. ``terngrad'' represents communicating with Terngrad \cite{wen2017terngrad}. ``com1'' represents communicating with example \ref{quantization-mapping1} and ``com2'' represents communicating with example \ref{quantization-mapping2}.}
    \label{Fig_imdb_bits}
%    \vspace{-0.7cm}
\end{figure*}

%\vspace{-0.5cm}
\subsection{Binary Sentiment Classification}
%\vspace{-0.1cm}
In addition, we apply our method to the sentiment classification task. We train a GRU \cite{cho2014properties} network on the dataset IMDB \cite{maas-EtAl:2011:ACL-HLT2011}.
The dataset contains 50,000 movie reviews which are labeled as positive or negative. Besides, the dataset has been split into two sets: one is a training set that contains 25,000 reviews and the other is a test set containing the rest 25,000 reviews. For the network part, we use a bidirectional GRU with 2 layers. In each layer, we set the hidden unit dimension as 256. Moreover, we use pertained Bert \cite{devlin2018bert} encoding to encode the text before it goes into GRU. During the training phase, we set iterations as 2000, parameter $\beta = 0.9$, and parameter $\theta_t = 0.999$ for Adam-based algorithms, and the momentum coefficient for sgd-based algorithm is 0.9.  The learning rate for each algorithm is selected from $\{1e-1,5e-2,1e-2,5e-3,1e-3,5e-4,1e-4\}$ via grid search approach.  The batch size is 16 and 8 workers are involved in the training phase. No regularization terms or learning rate decay are used in the training phase. The hyper-parameters for the quantization function are the same as it in Section \ref{convex}.

The experimental results are shown in Figures \ref{Fig_imdb_bits}, \ref{Fig_imdb}, and \ref{Fig_imdb_err}, where training loss has been smoothed for better presentation. In Figure \ref{Fig_imdb},  sgd-based algorithms perform worse than Adam-based algorithms. Still, the quantization of communication affects a little on our algorithm. Besides, Efficient-Adam with the quantizer in Example \ref{quantization-mapping2} even achieves the highest test accuracy among all compared algorithms. When considering training loss and test accuracy vs. communication, shown in Figure \ref{Fig_imdb_bits}, Efficient-Adam with Example \ref{quantization-mapping1} can achieve lower training loss and comparable high test-accuracy as Distributed Adam with Terngrad, while Efficient-Adam with Example \ref{quantization-mapping2} can achieve the highest test accuracy after 500MB communication.

\begin{figure}[htb!]
%\vspace{-0.4cm}
    \centering
    \!\!\!\!\includegraphics[width=0.24\textwidth]{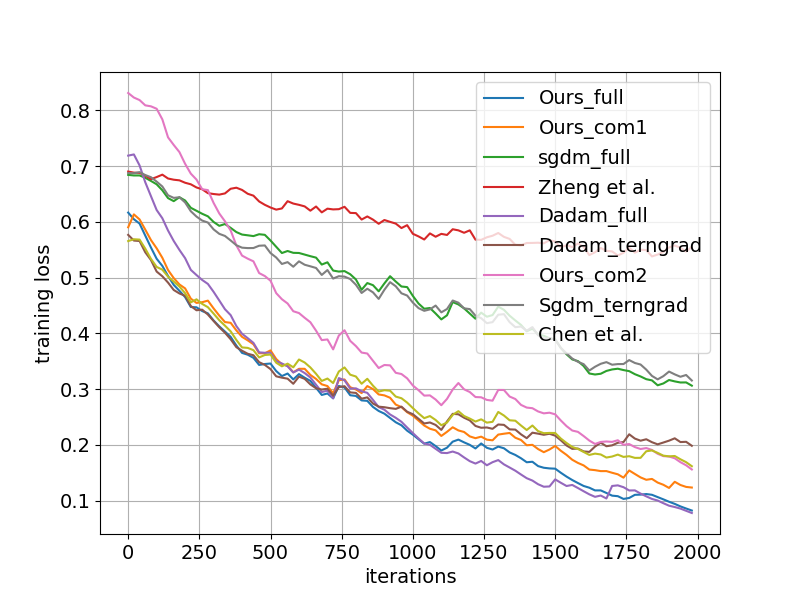}\!\!\!\!
    \!\!\!\!\includegraphics[width=0.24\textwidth]{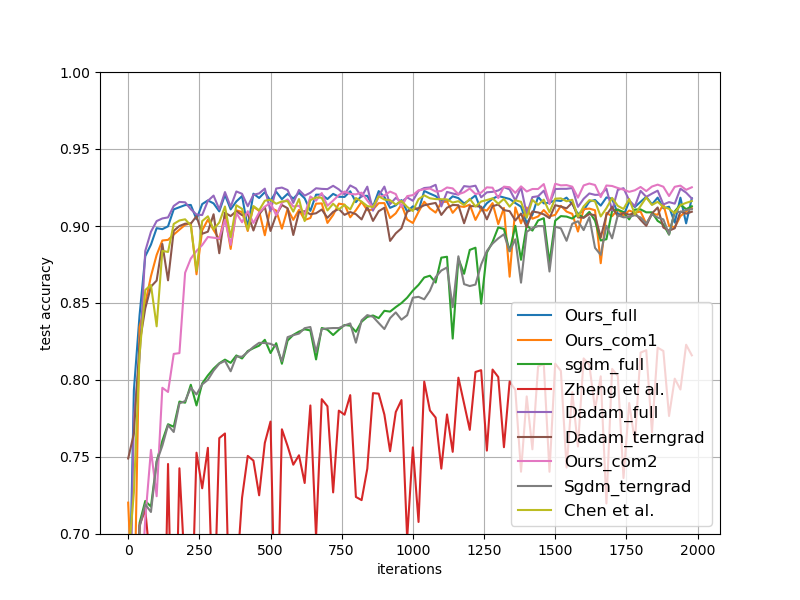}\!\!\!\!
%     \vspace{-0.4cm}
    \caption{\change{Experimental results for binary sentiment classification on the IMDB dataset. The first figure and the second figure show curves on training loss vs. iterations and test accuracy vs. iterations, respectively.}}% ``full'' represents full-precision communication. ``terngrad'' represents communicating with Terngrad \cite{wen2017terngrad}. ``com1'' represents communicating with example \ref{quantization-mapping1} and ``com2'' represents communicating with example \ref{quantization-mapping2}.}
    \label{Fig_imdb}
%    \vspace{-0.35cm}
\end{figure}

Moreover, Figure \ref{Fig_imdb_err} shows the influence of the error-feedback technique. Still, without error-feedback, Efficient-Adam with quantizer in Example \ref{quantization-mapping1} achieves much worse performance than the algorithm with error-feedback, and Efficient-Adam with quantizer in Example \ref{quantization-mapping2} achieves a slightly bad test accuracy or achieves even better training loss than the algorithm with error-feedback. Results can be concluded similarly as it is in Section \ref{convex} and \ref{cifar_exp}, where the more error quantization function will introduce, the more helpful error feedback is.

\begin{figure}[htb!]
% \vspace{-0.5cm}
   \centering
   \!\!\!\!\includegraphics[width=0.24\textwidth]{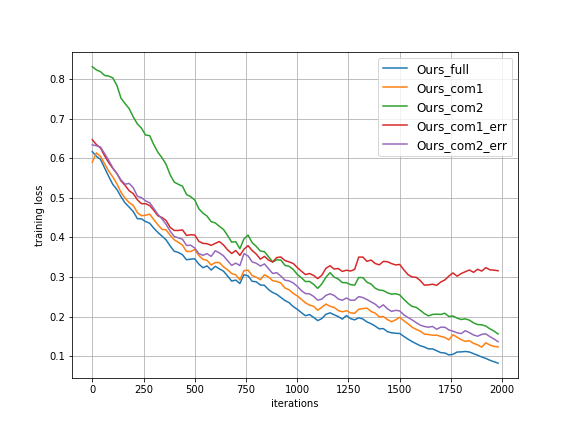}\!\!\!\!
   \!\!\!\!\includegraphics[width=0.24\textwidth]{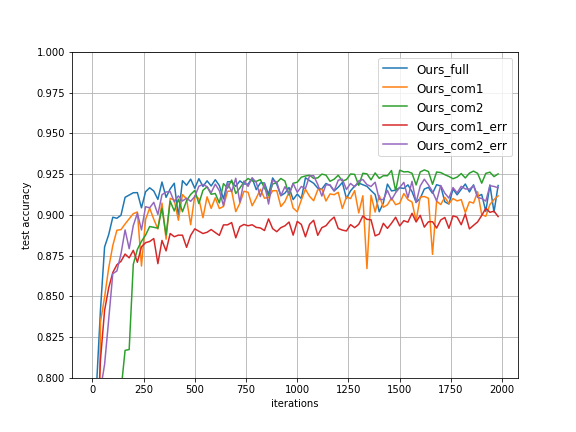}\!\!\!\!
 %    \vspace{-0.3cm}
    \caption{Ablation study with/without error-feedback on the training loss and test accuracy vs. iterations for binary sentiment classification task on the IMDB dataset, respectively. The left figure and right figure show curves on training loss vs. iterations and the curves on test accuracy vs. iterations, respectively.}% ``full'' represents full-precision communication. ``com1'' represents communicating with example \ref{quantization-mapping1} and ``com2'' represents communicating with Example \ref{quantization-mapping2}. ``err'' means running our algorithms without error-feedback technique.}
%    \vspace{-0.5cm}
    \label{Fig_imdb_err}
\end{figure}

\section{Conclusion}\label{conclusion-sec}

We proposed a communication efficient adaptive stochastic gradient descent algorithm for optimizing the non-convex stochastic problem, dubbed {\em Efficient-Adam}. In the algorithm, we used a two-side quantization to reduce the communication overhead and two error feedback terms to compensate for the quantization error to encourage convergence. With the more practical assumptions, we established a theoretical convergence result for the algorithm. Besides, we established the communication complexity and iteration complexity under certain quantization functions. On the other hand, when the quantization operators are generalized to compressors, Efficient-Adam can achieve the same convergence rate as full-precision Adam. Lastly, we applied the algorithm to a toy task and real-world image and sentiment classification tasks. The experimental results confirm the efficacy of the proposed Efficient-Adam. However, we merely established the convergence of the Efficient-Adam.
\change{To demonstrate the linear-speedup characteristic of Efficient-Adam in a distributed environment, alternative assumptions, such as bounding the variance of stochastic gradients, may be necessary. We identify this as a potential area for future exploration. Additionally, the algorithm could potentially be adapted into an asynchronous version, which would involve assessing the impact of delays on the algorithm's performance. We also earmark this aspect for future investigation.}%To show the linear-speedup property of Efficient-Adam in the distributed setting, alternative assumptions such as the bounded variance of stochastic gradients are needed, which is left as future work. 

%\section*{Acknowledgment}
%\newpage

\bibliographystyle{IEEEtran}
\bibliography{example_paper}

\clearpage
\appendices
\vspace{-0.4cm}
\section{Proof of Efficient-Adam}\label{proof-efficient-adam1}% with Quantization functions}
\vspace{-0.1cm}
To prove Theorem \ref{complexity-quantization}, more notations and lemmas are needed. Below, we introduce several notations and lemmas.

 \vspace{-0.15cm}
\begin{notation}\label{notation-thm1}
Let \small{$\mathbb{E}_t\left(\cdot\right) = \mathbb{E}_\xi\left(\cdot|x_t, v_{t-1}, m_{t-1}\right)$. }
Denote 
\small{$\sigma_t^2 = \mathbb{E}_t\left(g_t^2\right)$, 
$\hat{v}^{\left(i\right)}_t = \theta_t v^{\left(i\right)}_t + \left(1-\theta_t\right) {\sigma}^2_t$, 
$\hat{\eta}^{\left(i\right)}_t = \alpha_t/\sqrt{\hat{v}^{\left(i\right)}_t}$, 
$\Delta_t^{\left(i\right)} = -\alpha_t m_t^{\left(i\right)}/\sqrt{v_t^{\left(i\right)}}$,
$\Delta_t = \frac{1}{N}\sum_{i=1}^N \Delta_t^{\left(i\right)}$, 
$E_t = \frac{1}{N}\sum_{i=1}^N e_t^{\left(i\right)}$, 
$M_t = \mathbb{E}[\langle \nabla f\left(x_{t}\right),\Delta_{t}\rangle + \frac{L\left(2-\delta_s\right)}{N}\sum_{i=1}^N \|e_{t}\|\|\Delta_{t}^{\left(i\right)}\| + \frac{L\left(2-\delta_s\right)\left(2-\delta_w\right)}{N}\sum_{i=1}^N \|\Delta_{t}^{\left(i\right)}\|^2 $ $+ \frac{L\left(2-\delta_s\right)\left(2-\delta_w\right)}{N^2}\sum_{i=1}^N\sum_{j=1}^N\|e_{t}^{\left(i\right)}\|\|\Delta_{t}^{\left(j\right)}\|]$, and
$\theta' = 1-\theta/T$.}
In addition, let $\tilde{x}_t = x_t - e_{t}$ and  $ \hat{x_t} = \tilde{x}_t - E_t$. 
\end{notation}
\vspace{-0.15cm}

Using the above notations, it is not hard to check that the following two equations hold:

\small{\begin{gather}
\tilde{x}_{t+1} = x_t - \mathcal{Q}_s\left(\hat{\delta}_t+e_t\right) - e_{t+1} = x_t + \hat{\delta}_t -e_t = \tilde{x}_t + \hat{\delta}_t, \\
\hat{x}_{t+1} = \tilde{x_t} - \frac{1}{N}\sum_{t=1}^N \mathcal{Q}_w\left(-\Delta_t^{\left(i\right)} +e_t^{\left(i\right)}\right) - E_{t+1}= \hat{x}_t + \Delta_t.
%= \tilde{x_t} - \frac{1}{N}\sum_{t=1}^N -\Delta_t^{\left(i\right)} - E_{t} .
\end{gather}}
\vspace{-0.3cm}
\begin{lemma}\label{Lemmamv}
By using Notation \ref{notation-thm1} and the iteration scheme of algorithm \ref{alg2}, $\forall\ t\ge 1$ and $i=1,2,\cdots,N$, the following inequality holds:
\[\small
  \left(m_t^{\left(i\right)}\right)^2\leq \frac{1}{\left(1-\gamma\right)\left(1-\theta_t\right)}v_t^{\left(i\right)}.
\]
\end{lemma}
\begin{proof}
By using the definitions of $m_t^{\left(i\right)}$ and $v_t^{\left(i\right)}$ in Algorithm \ref{alg2}, it directly holds that

\small{\[
\begin{split}
m_t ^{\left(i\right)}= \sum_{k=1}^t\beta^{t-k}\left(1-\beta\right)g_k^{\left(i\right)},
 v_t^{\left(i\right)} = \sum_{k = 1}^t \left(\theta'\right)^{t-k}\left(1-\theta_k\right)\left(g_k^{\left(i\right)}\right)^2.
 \end{split}
 \]}
With arithmetic inequality, it holds that

\small{\[
\begin{split}
    \left(m_t^{\left(i\right)}\right)^2 
    &= \left(\sum_{k=1}^t\frac{\beta^{t-k}\left(1-\beta\right)}{\sqrt{\theta'^{t-k}\left(1-\theta_k\right)}}\sqrt{{\theta'}^{t-k} \left(1-\theta_k\right)}g_k^{\left(i\right)} \right)^2 \\
    &\leq \sum_{k=1}^t\frac{\beta^{2\left(t-k\right)}\left(1-\beta\right)^2}{\theta'^{t-k}\left(1-\theta_k\right)}\sum_{k=1}^t\theta'^{t-k}\left(1-\theta_k\right)\left(g_k^{\left(i\right)}\right)^2\\
    &= \sum_{k=1}^t \frac{\beta^{2\left(t-k\right)}\left(1-\beta\right)^2}{\theta'^{t-k}\left(1-\theta_k\right)}v_t^{\left(i\right)}
    \leq \sum_{i=1}^t\frac{\beta^{2\left(t-i\right)}}{\left(\theta'\right)^{t-k}\left(1-\theta_k\right)}v_t^{\left(i\right)} \\
    &\leq \frac{1}{\left(1-\gamma\right)\left(1-\theta_t\right)}v_t.
\end{split}
\]}

Then, we obtain the desired result. 
\end{proof}

\begin{lemma}\label{esum}
Let $e^{\left(i\right)}_t$ be the noisy term defined in Algorithm \ref{alg2} and $\Delta^{\left(i\right)}_t$ be the term defined in Notation \ref{notation-thm1}. Then it holds that 

\small{\begin{align*}
&\mathbb{E}\left[\frac{1}{N^2}\sum_{t=1}^T\sum_{i=1}^N\sum_{j=1}^N \|e_t^{\left(i\right)}\|\|\Delta_t^{\left(j\right)}\|\right] \leq\\ &\frac{1-\delta_w}{\delta_wN}\sum_{t=1}^T\sum_{i=1}^N \mathbb{E}\|\Delta_t^{\left(i\right)}\|^2 + \frac{\delta_w'}{\delta_w N} \sum_{t=1}^T\sum_{i=1}^N \mathbb{E}\|\Delta_t^{\left(i\right)}\|.
\end{align*}}

\end{lemma}
\begin{proof}
Using the definition of the noisy term  $e_{t}$, the following holds:
\small{\begin{align*}
    \|e^{\left(j\right)}_t\| 
    &= \|\Delta_{t-1}^{\left(j\right)} + e_{t-1}^{\left(j\right)} - \mathcal{Q}_w\left(\Delta_{t-1}^{\left(j\right)} + e_{t-1}^{\left(j\right)}\right)\|\\
    &\leq \left(1-\delta_w\right) \left(\|\Delta_{t-1}^{\left(j\right)}\| + \|e_{t-1}^{\left(j\right)}\|\right) + \delta_w' \\
    %&\leq \left(1-\delta_w\right) \|\Delta_{t-1}^{\left(j\right)}\| \\
    %&\qquad + (1-\delta_w)\left(\left(1-\delta_w\right)\left(\|\Delta_{t-2}^{\left(j\right)}\| + \|e_{t-2}^{\left(j\right)}\|\right) + \delta_w'\right) + \delta_w'\\
    &\leq \cdots \leq \sum_{k = 1}^{t-1} \left(1-\delta_w\right)^{t-k} \|\Delta_k^{\left(j\right)}\| + \frac{\delta_w'}{\delta_w}.
\end{align*}}
In addition, 
\small{\[
\begin{split}
&\frac{1}{N^2}\sum_{t=1}^T\sum_{i=1}^N\sum_{j=1}^N \|e_t^{\left(i\right)}\|\|\Delta_t^{\left(j\right)}\|\\ 
%&\leq \frac{1}{N^2}\sum_{t=1}^T\sum_{i=1}^N\sum_{j=1}^N\left(\sum_{k=1}^{t-1} \left(1-\delta_w\right)^{t-k}\|\Delta_k^{\left(i\right)}\| + \frac{\delta_w'}{\delta_w}\right)\|\Delta_t^{\left(j\right)}\|\\
%&\leq\frac{1}{N^2}\sum_{t=1}^T\sum_{i=1}^N\sum_{j=1}^N\sum_{k=1}^{t-1} \left(1-\delta_w\right)^{t-k}\|\Delta_k^{\left(i\right)}\|\|\Delta_t^{\left(j\right)}\|\\
%& \qquad  + \frac{\delta_w'}{\delta_w N} \sum_{t=1}^T\sum_{i=1}^N \mathbb{E}\|\Delta_t^{\left(i\right)}\|   \\
&\leq \frac{1}{N^2}\sum_{t=1}^T\sum_{i=1}^N\sum_{j=1}^N\sum_{k=1}^{t-1} \frac{\left(1-\delta_w\right)^{t-k}}{2} \left(\|\Delta_k^{\left(i\right)}\|^2+ \|\Delta_t^{\left(j\right)}\|^2\right)\\
& \qquad + \frac{\delta_w'}{\delta_w N} \sum_{t=1}^T\sum_{i=1}^N \|\Delta_t^{\left(i\right)}\|   \\
%&= \frac{1}{2N}\sum_{t=1}^T\sum_{i=1}^N\sum_{k=1}^{t-1} \left(1-\delta_w\right)^{t-k}\|\Delta_k^{\left(i\right)}\|^2  + \frac{\delta_w'}{\delta_w N} \sum_{t=1}^T\sum_{i=1}^N \|\Delta_t^{\left(i\right)}\| 
%\\ &\qquad + \frac{1}{2N}\sum_{t=1}^T\sum_{j=1}^N\sum_{k=1}^{t-1} \left(1-\delta_w\right)^{t-k} \|\Delta_t^{\left(j\right)}\|^2 \\
%&\leq \frac{1-\delta_w}{2\delta_wN}\sum_{k=1}^T\sum_{i=1}^N \|\Delta_k^{\left(i\right)}\|^2 + \frac{1-\delta_w}{2\delta_wN}\sum_{t=1}^T\sum_{j=1}^N \mathbb{E}\|\Delta_t^{\left(j\right)}\|^2  \\
%& \qquad + \frac{\delta_w'}{\delta_w N} \sum_{t=1}^T\sum_{i=1}^N \|\Delta_t^{\left(i\right)}\|  \\
&\leq \frac{1-\delta_w}{\delta_wN}\sum_{t=1}^T\sum_{i=1}^N \|\Delta_t^{\left(i\right)}\|^2 + \frac{\delta_w'}{\delta_w N} \sum_{t=1}^T\sum_{i=1}^N \|\Delta_t^{\left(i\right)}\|.
\end{split}
\]
}
Then, by taking expectations on both sides of the above inequality, we get the desired result.
\end{proof}

\begin{lemma}\label{eesum}
Let $e_t$ be the noisy term defined in Algorithm \ref{alg1} and $\Delta^{\left(i\right)}_t$ be the term defined in Notation \ref{notation-thm1}. Then, it holds that 
\small{\[
\begin{split}
&\mathbb{E}\left[\frac{1}{N}\sum_{t=1}^T\sum_{i=1}^N\|e_t\|\|\Delta_t^{\left(i\right)}\|\right]\\
&\leq \mathbb{E}\left[\frac{(2-\delta_w)\left(1-\delta_s\right)}{N\delta_s\delta_w}\sum_{t=1}^T\sum_{i=1}^N\|\Delta_t^{\left(i\right)}\|^2 \right. \\
& \left. \qquad + \frac{1}{N}\left(\frac{\delta_s'}{\delta_s} + \frac{(2-\delta_w)\left(1-\delta_s\right)\delta_w'}{\delta_s\delta_w}\right)\sum_{t=1}^T\sum_{i=1}^N\|\Delta_t^{\left(i\right)}\|\right].
\end{split}
\]}
\end{lemma}
\begin{proof}
By the definition of $\hat{\delta_t}$ in Algorithm \ref{alg1} and Algorithm \ref{alg2}, it holds that
\small{\[
\|\hat{\delta}_t\| \!= \!\frac{1}{N}\!\sum_{i=1}^N\! \left\|Q_w\left(\Delta_t^{(i)}\! +\! e_t^{(i)}\right)\right\|\! \leq\! \frac{2-\delta_w}{N} \!\sum_{i=1}^N\! \|\Delta_t^{(i)}\| + \|e_t^{(i)}\|.\]}

With the definition of $e_t$, we obtain that

\small{\[
\begin{split}
&\|e_t\| = \|\hat{\delta}_{t-1} + e_{t-1} - \mathcal{Q}_s\left(\hat{\delta}_{t-1} + e_{t-1}\right)\|\\
&\leq \left(1-\delta_s\right)\left(\|\hat{\delta}_{t-1}\| + \|e_{t-1}\|\right) +\delta_s'\\
 &\leq \left(1-\delta_s\right)\left(\frac{(2-\delta_w)}{N}\sum_{j=1}^N\left(\|\Delta_{t-1}^{\left(j\right)}\| + \|e_{t-1}^{\left(j\right)}\|\right) + \|e_{t-1}\|\right) +\delta_s' \\
%&\leq \frac{\left(1-\delta_s\right)(2-\delta_w)}{N}\left(\sum_{j=1}^N\|\Delta_{t-1}^{\left(j\right)}\| + \sum_{j=1}^N\sum_{k=1}^{t-1}\left(1-\delta_w\right)^{t-k}\|\Delta_k^{\left(j\right)}\|\right) \\
%&\qquad + \frac{(2-\delta_w)(1-\delta_s)\delta_w'}{\delta_w} +(1-\delta_s)\|e_{t-1}\| + \delta_s'\\
&\leq \frac{2-\delta_w}{N}\sum_{l=1}^{t-1} \left(1-\delta_s\right)^{t-l}\sum_{j=1}^N\|\Delta_l^{\left(j\right)}\|+ \frac{\delta_s'}{\delta_s} + \frac{(2-\delta_w)\left(1-\delta_s\right)\delta_w'}{\delta_s\delta_w}\\
& \qquad + \frac{2-\delta_w}{N}\sum_{l=1}^{t-1} \left(1-\delta_s\right)^{t-l}
\sum_{j=1}^N\sum_{k = 1}^{l-1} \left(1-\delta_w\right)^{l-k}\|\Delta_k^{\left(j\right)}\|
%&\qquad .
\end{split} 
\]}

Plugging $\|e_t\|$ into the left-hand-side and summing over $T$ and $N$, we obtain that

{\small \begin{equation}
\label{lemma3_1}    
\begin{split}
&\frac{1}{N} \sum_{t=1}^T\sum_{i=1}^N\|e_t\|\|\Delta_t^{\left(i\right)}\|\\
&\leq \frac{2-\delta_w}{2N^2}\sum_{t=1}^T\sum_{i=1}^N\sum_{l=1}^{t-1}\sum_{j=1}^N\left(1-\delta_s\right)^{t-l}\left(\|\Delta_l^{\left(j\right)}\|^2+ \|\Delta_t^{\left(i\right)}\|^2\right)\\
& + \frac{2-\delta_w}{2N^2}\sum_{t=1}^T\sum_{i=1}^N\sum_{l=1}^{t-1}\sum_{j=1}^N\left(1-\delta_s\right)^{t-l} \sum_{k=1}^{l-1}\left(1-\delta_w\right)^{l-k}\|\Delta_k^{\left(j\right)}\|^2 \\
&+ \frac{2-\delta_w}{2N^2}\sum_{t=1}^T\sum_{i=1}^N\sum_{l=1}^{t-1}\sum_{j=1}^N\left(1-\delta_s\right)^{t-l} \sum_{k=1}^{l-1}\left(1-\delta_w\right)^{l-k} \|\Delta_t^{\left(i\right)}\|^2\\
& + \frac{1}{N}\left(\frac{\delta_s'}{\delta_s} + \frac{(2-\delta_w)\left(1-\delta_s\right)\delta_w'}{\delta_s\delta_w}\right)\sum_{t=1}^T\sum_{i=1}^N\|\Delta_t^{\left(i\right)}\|. \\
\end{split}
\end{equation}
}
Then we simplify the above summation terms, for the first term, it holds that
{\small \begin{equation}
\label{lemma3_2}
\begin{split}
    &\frac{2-\delta_w}{2N^2}\sum_{t=1}^T\sum_{i=1}^N\sum_{l=1}^{t-1}\sum_{j=1}^N\left(1-\delta_s\right)^{t-l}\left(\|\Delta_l^{\left(j\right)}\|^2+ \|\Delta_t^{\left(i\right)}\|^2\right)\\
    %&=  \frac{2-\delta_w}{2N}\sum_{t=1}^T\sum_{l=1}^{t-1}\sum_{j=1}^N\left(1-\delta_s\right)^{t-l} \|\Delta_l^{\left(j\right)}\|^2\\
    %&+  \frac{2-\delta_w}{2N}\sum_{t=1}^T\sum_{i=1}^N\sum_{l=1}^{t-1}\left(1-\delta_s\right)^{t-l}\|\Delta_t^{\left(i\right)}\|^2\\
    &\leq \frac{(2-\delta_w)(1-\delta_s)}{2N\delta_s}\sum_{l=1}^{T}\sum_{j=1}^N \|\Delta_l^{\left(j\right)}\|^2+ \frac{(2-\delta_w)(1-\delta_s)}{2N\delta_s}\sum_{t=1}^T\sum_{i=1}^N\|\Delta_t^{\left(i\right)}\|^2\\
    &= \frac{(2-\delta_w)(1-\delta_s)}{N\delta_s}\sum_{t=1}^{T}\sum_{i=1}^N \|\Delta_t^{\left(i\right)}\|^2.
\end{split}
\end{equation}
}
For the second summation term, we have
{\small
\begin{equation}
    \label{lemma3_3}
    \begin{split}
     &\frac{2-\delta_w}{2N^2}\sum_{t=1}^T\sum_{i=1}^N\sum_{l=1}^{t-1}\sum_{j=1}^N\left(1-\delta_s\right)^{t-l} \sum_{k=1}^{l-1}\left(1-\delta_w\right)^{l-k}\|\Delta_k^{\left(j\right)}\|^2\\
     %&= \frac{2-\delta_w}{2N}\sum_{t=1}^T\sum_{l=1}^{t-1}\sum_{j=1}^N\left(1-\delta_s\right)^{t-l} \sum_{k=1}^{l-1}\left(1-\delta_w\right)^{l-k}\|\Delta_k^{\left(j\right)}\|^2\\
     &\leq \frac{2-\delta_w(1-\delta_w)(1-\delta_s)}{2N \delta_w\delta_s}\sum_{k=1}^T\sum_{j=1}^{N}\|\Delta_k^{\left(j\right)}\|^2.
\end{split}
\end{equation}
}

And for the third term, it holds that
{\small \begin{equation}
\label{lemma3_4}
\begin{split}
   &\frac{2-\delta_w}{2N^2}\sum_{t=1}^T\sum_{i=1}^N\sum_{l=1}^{t-1}\sum_{j=1}^N\left(1-\delta_s\right)^{t-l} \sum_{k=1}^{l-1}\left(1-\delta_w\right)^{l-k} \|\Delta_t^{\left(i\right)}\|^2\\
   %& = \frac{2-\delta_w}{2N}\sum_{t=1}^T\sum_{i=1}^N\sum_{l=1}^{t-1}\left(1-\delta_s\right)^{t-l} \sum_{k=1}^{l-1}\left(1-\delta_w\right)^{l-k} \|\Delta_t^{\left(i\right)}\|^2\\
   &\leq \frac{(2-\delta_w)(1-\delta_w)(1-\delta_s)}{2N\delta_w\delta_s}\sum_{t=1}^T\sum_{i=1}^N\|\Delta_t^{\left(i\right)}\|^2
\end{split}
\end{equation}
}

Combining inequalities \eqref{lemma3_1}, \eqref{lemma3_2}, \eqref{lemma3_3} and \eqref{lemma3_4}, it holds that

{\small \[
\begin{split}
&\frac{1}{N}\sum_{t=1}^T\sum_{i=1}^N \|e_t\|\|\Delta_t^{(i)}\|\\
%&\leq \frac{(2-\delta_w)(1-\delta_s)}{N\delta_s}\sum_{t=1}^{T}\sum_{i=1}^N \|\Delta_t^{\left(i\right)}\|^2 \\
%& \qquad + \frac{2-\delta_w(1-\delta_w)(1-\delta_s)}{2N\delta_w\delta_s}\sum_{k=1}^T\sum_{j=1}^{N}\|\Delta_k^{\left(j\right)}\|^2\\
%& \qquad +  \frac{2-\delta_w(1-\delta_w)(1-\delta_s)}{2N\delta_w\delta_s}\sum_{k=1}^T\sum_{j=1}^{N}\|\Delta_k^{\left(j\right)}\|^2\\
%& \qquad + \frac{1}{N}\left(\frac{\delta_s'}{\delta_s} + \frac{(2-\delta_w)\left(1-\delta_s\right)\delta_w'}{\delta_s\delta_w}\right)\sum_{t=1}^T\sum_{i=1}^N\|\Delta_t^{\left(i\right)}\|\\
& \leq \frac{(2-\delta_w)\left(1-\delta_s\right)}{N\delta_s\delta_w}\sum_{t=1}^T\sum_{i=1}^N\|\Delta_t^{\left(i\right)}\|^2 \\
& \qquad + \frac{1}{N}\left(\frac{\delta_s'}{\delta_s} + \frac{(2-\delta_w)\left(1-\delta_s\right)\delta_w'}{\delta_s\delta_w}\right)\sum_{t=1}^T\sum_{i=1}^N\|\Delta_t^{\left(i\right)}\|.
\end{split}
\]
}

Then, by taking expectations on both sides of the above inequality, we get the desired result.
\end{proof}

\begin{lemma}\label{sumdelta}
By using Notation\ref{notation-thm1}, for any $i = 1,2,\cdots,N$, the following inequality holds:
{\small \[
\sum_{t=1}^T \|\Delta_t^{\left(i\right)}\|^2 \leq \frac{\alpha^2}{\theta\left(1-\sqrt{\gamma}\right)^2}\sum_{t=1}^T\left\|\frac{\sqrt{1-\theta_t}g_t^{\left(i\right)}}{\sqrt{v_t^{\left(i\right)}}}\right\|^2.
\]}
\end{lemma}
\begin{proof}
By the definition of $m_t^{\left(i\right)}$, we can obtain
{\small \[
|m_t^{\left(i\right)}| = \left|\sum_{i=1}^t \beta^{t-i}\left(1-\beta\right)g_i^{\left(i\right)}\right|\leq \sum_{i=1}^t \beta^{t-i}|g_i^{\left(i\right)}|.
\]}

Because $v_t^{\left(i\right)} \geq \theta' v_{t-1}^{\left(i\right)}$, we can easily obtain that 
$v_t^{\left(i\right)} \geq \left(\theta'\right)^{t-k}  v_k^{\left(i\right)}$.

By using the definition of $\Delta_t^{\left(i\right)}$, we can obtain
{\small\[
\begin{split}
\|\Delta_t^{\left(i\right)}\|^2 &= \left\|\frac{\alpha_t m_t^{\left(i\right)}}{\sqrt{v_t^{\left(i\right)}}}\right\|^2
\leq \alpha_t^2 \sum_{k=1}^t \left\|\frac{\beta^{t-k}|g_k^{\left(i\right)}|}{\sqrt{v_t^{\left(i\right)}}}\right\|^2\\
& \leq \frac{\alpha_t^2}{\left(1-\theta_t\right)} \sum_{k=1}^t\left\|\left(\frac{\beta}{\sqrt{\theta'}}\right)^{t-i} \frac{\sqrt{1-\theta_t} |g_k^{\left(i\right)}|}{\sqrt{v_t^{\left(i\right)}}}\right\|^2\\
%&\leq \frac{\alpha^2}{\theta}\sum_{k=1}^t\left\|\sqrt{\gamma}^{t-k}\frac{\sqrt{1-\theta_k}|g_k^{\left(i\right)}|}{\sqrt{v_k^{\left(i\right)}}}\right\|^2 \\
&\leq \frac{\alpha^2}{\theta}\left(\sum_{k=1}^t\sqrt{\gamma}^{t-k}\right)\sum_{k=1}^t \sqrt{\gamma}^{t-k}\left\|\frac{\sqrt{1-\theta_k}g_k^{\left(i\right)}}{\sqrt{v_k^{\left(i\right)}}}\right\|^2\\
&\leq \frac{\alpha^2}{\theta\left(1-\sqrt{\gamma}\right)}\sum_{k=1}^t \sqrt{\gamma}^{t-k}\left\|\frac{\sqrt{1-\theta_k}g_k^{\left(i\right)}}{\sqrt{v_k^{\left(i\right)}}}\right\|^2.
\end{split}
\]}

Hence, it follows that
{\small\[
\begin{split}
\sum_{t=1}^T\|\Delta_t^{\left(i\right)}\|^2 &\leq \frac{\alpha^2}{\theta\left(1-\sqrt{\gamma}\right)}\sum_{t=1}^T\sum_{k=1}^t \sqrt{\gamma}^{t-k}\left\|\frac{\sqrt{1-\theta_k}g_k^{\left(i\right)}}{\sqrt{v_k^{\left(i\right)}}}\right\|^2 \\
%& = \frac{\alpha^2}{\theta\left(1-\sqrt{\gamma}\right)}\sum_{k=1}^T \sum_{t=k}^T \sqrt{\gamma}^{t-k} \left\|\frac{\sqrt{1-\theta_k}g_k^{\left(i\right)}}{\sqrt{v_k^{\left(i\right)}}}\right\|^2\\
&\leq \frac{\alpha^2}{\theta\left(1-\sqrt{\gamma}\right)^2}\sum_{t=1}^T \left\|\frac{\sqrt{1-\theta_t}g_t^{\left(i\right)}}{\sqrt{v_t^{\left(i\right)}}}\right\|^2.
\end{split}
\]}
Thus, the desired result is obtained.
\end{proof}

\begin{lemma}\label{sumgv}
For all $i=1,2,\cdots,N$, the following estimate holds:
{\small \[
\mathbb{E}\left[\sum_{t=1}^T\left\|\frac{\sqrt{1-\theta_t}g_t^{\left(i\right)}}{\sqrt{v_t^{\left(i\right)}}}\right\|^2\right] \leq d\left[\log\left(1+\frac{G^2}{\epsilon d}\right)+\frac{\theta}{1-\theta}\right].
\]}
\end{lemma}

\begin{proof}
Let $W_0 = 1$ and $W_t = \prod_{i=1}^t \theta_i^{-1}$. Let $w_t = W_t-W_{t-1} = \left(1-\theta_t\right)W_t$. Therefore, we have $\frac{w_t}{W_t} = 1-\theta_t$, and $\frac{W_{t-1}}{W_t} = \theta_t$.
By observing that $W_0v_0^{\left(i\right)} = \epsilon$ and {\small $W_tv_t^{\left(i\right)} = W_{t-1}v_{t-1}^{\left(i\right)} + w_t \left(g_t^{\left(i\right)}\right)^2$}, then we can obtain 
{\small $W_tv_t^{\left(i\right)} = \epsilon + \sum_{k=1}^tw_k\left(g_k^{\left(i\right)}\right)^2$.}

By plugging the equality into the left-hand side, we obtain
{\small\[
\begin{split}
&\sum_{t=1}^T\left\|\frac{\sqrt{1-\theta_t} g_t^{\left(i\right)}}{\sqrt{v_t^{\left(i\right)}}}\right\|^2 = \sum_{t=1}^T \left\|\frac{\left(1-\theta_t\right)\left(g_t^{\left(i\right)}\right)^2}{v_t^{\left(i\right)}}\right\|_1\\
&= \sum_{t=1}^T \left\|\frac{w_t\left(g_t^{\left(i\right)}\right)^2}{W_tv_t^{\left(i\right)}}\right\|_1 = \sum_{t=1}^T\left\|\frac{w_t\left(g_t^{\left(i\right)}\right)^2}{\epsilon + \sum_{k=1}^t w_k\left(g_k^{\left(i\right)}\right)^2}\right\|_1\\
& =\sum_{t=1}^T \sum_{j=1}^d  \frac{w_t\left(g_{t,j}^{\left(i\right)}\right)^2}{\epsilon+\sum_{k=1}^t w_i\left(g_{k,j}^{\left(i\right)}\right)^2} = \sum_{j=1}^d\sum_{t=1}^T \frac{w_tg_{t,j}^2}{\epsilon+\sum_{k=1}^t w_k\left(g_{k,j}^{\left(i\right)}\right)^2}.
\end{split}
\] }

Considering each dimension respectively, for each $j=1,2,\cdots,d$, we obtain
{\small\[
\begin{split}
 &\sum_{t=1}^T \frac{w_t\left(g_{t,j}^{\left(i\right)}\right)^2}{\epsilon+\sum_{k=1}^t w_i\left(g_{k,j}^{\left(i\right)}\right)^2} \leq \log\left(\epsilon + \sum_{k=1}^tw_k\left(g_{k,j}^{\left(i\right)}\right)^2\right) - \log\left(\epsilon\right) \\
 &= \log\left(1 + \frac{1}{\epsilon}\sum_{k=1}^tw_i\left(g_{k,j}^{\left(i\right)}\right)^2\right).
 \end{split}
\]}

Therefore, 
{\small \[
\begin{split}
&\sum_{t=1}^T\left\|\frac{\sqrt{1-\theta_t} g_t^{\left(i\right)} }{\sqrt{v_t^{\left(i\right)}}}\right\|^2 \leq \sum_{j=1}^d \log\left(1 + \frac{1}{\epsilon}\sum_{k=1}^Tw_k\left(g_{k,j}^{\left(i\right)}\right)^2\right) \\
%&\leq  d\log\left(1 + \frac{1}{d\epsilon}\sum_{j=1}^d \sum_{k=1}^Tw_k\left(g_{k,j}^{\left(i\right)}\right)^2\right)\\
&\leq d\log\left(1 + \frac{1}{d\epsilon}\sum_{k=1}^Tw_k\|g_k^{\left(i\right)}\|^2\right)\\
&\leq d\log\left(1 + \frac{G^2}{d\epsilon}\left(W_T-W_0\right)\right)\\
%&=d\log\left(1 + \frac{G^2}{d\epsilon}\left(\prod_{i=1}^T \theta_i^{-1} -1\right)\right)\\
%&\leq d\left[\log\left(1+\frac{G^2}{d\epsilon}\right) + \sum_{t=1}^T\log\left(\theta_t^{-1}\right)\right]\\
%&\leq d\left[\log\left(1+\frac{G^2}{d\epsilon}\right) + \sum_{t=1}^T \left(\theta_t^{-1} -1\right)\right] \\
%& \leq d\left[\log\left(1+\frac{G^2}{d\epsilon}\right) + \frac{1}{1-\theta}\sum_{t=1}^T \left(1-\theta_t\right)\right]\\
&\leq d\left[\log\left(1+\frac{G^2}{d\epsilon}\right) + \frac{\theta}{1-\theta}\right].
\end{split}
\]}
Hence, we obtain the desired result.
\end{proof}

\begin{lemma}
\label{sumnormdelta}
By the definition of $\Delta_t^{\left(i\right)}$, the following inequality always holds:
{\small \[
\mathbb{E} \sum_{t=1}^T \|\Delta_t^{\left(i\right)}\| \leq \sqrt{T} \sqrt{\frac{\alpha^2d}{\theta\left(1-\sqrt{\gamma}\right)^2}\left[\log\left(1+\frac{G^2}{\epsilon d}\right)+\frac{\theta}{1-\theta}\right]}.
\]}
\end{lemma}
\begin{proof}
Because of the concavity of function $\sqrt{x}$, we have {\small $\sum_{t=1}^T \|\Delta_t\|\leq \sqrt{T\sum_{t=1}^T\|\Delta_t\|^2}$}.
Hence, using Lemma \ref{sumdelta} and Lemma \ref{sumgv}, we obtain the desired result.
\end{proof}

\begin{lemma}\label{Mtmulbit}
With the definition of $M_t$, it holds that 
{\small\[
\begin{split}
\sum_{t=1}^T M_t \leq C_{2} - \frac{1-\beta}{2N}\sum_{t=1}^T\sum_{i=1}^N\mathbb{E}\|\nabla f\left(x_t\right)\|_{\hat{\eta}_t^{\left(i\right)}}^2 \\
\qquad + LC_3\alpha\sqrt{T}\left(\frac{\delta_s'}{\delta_s} +\frac{(2-\delta_w)\delta_w'}{\delta_s\delta_w}\right),
\end{split}
\]}
where 
{\small \[
\begin{split}
C_1 &= \left(\frac{\beta/\left(1-\beta\right)}{\sqrt{C_1\left(1-\gamma\right)\theta_1}}+1\right)^2,\\
C_{2} &= \frac{d}{1-\sqrt{\gamma}}\left(\frac{(2-\delta_w)(2-\delta_s)\alpha^2L}{\theta\left(1-\sqrt{\gamma}\right)^2\delta_w\delta_s} + \frac{2C_2G\alpha}{\sqrt{\theta}} \right)\\&\qquad \left[\log\left(1+\frac{G^2}{\epsilon d}\right)+\frac{\theta}{1-\theta}\right],\\
C_3 &=\sqrt{\frac{d}{\theta\left(1-\sqrt{\gamma}\right)^4}\left[\log\left(1+\frac{G^2}{\epsilon d}\right)+\frac{\theta}{1-\theta}\right]}.
\end{split}
\]}
\end{lemma}
\begin{proof}
Using the definitions of $v_t^{\left(i\right)}$ and $\hat{v}_t^{\left(i\right)}$, we obtain
\small{\[
\begin{split}
&\quad \frac{\left(1-\beta\right)\alpha_tg_t^{\left(i\right)}}{\sqrt{v_t^{\left(i\right)}}} \\
&= \frac{\left(1-\beta\right)\alpha_tg_t^{\left(i\right)}}{\sqrt{\hat{v}_t^{\left(i\right)}}} + \left(1-\beta\right)\alpha_tg_t^{\left(i\right)}\left(\frac{1}{\sqrt{v_t^{\left(i\right)}}} - \frac{1}{\sqrt{\hat{v}_t^{\left(i\right)}}}\right)\\
%&\!= \!\left(1-\beta\right)\hat{\eta}_t^{\left(i\right)}g_t^{\left(i\right)} + \left(1-\beta\right)\alpha_t\frac{\left(1-\theta_t\right)g_t^{\left(i\right)}\left(\sigma_t^2-\left(g_t^{\left(i\right)}\right)^2\right)}{\sqrt{v_t^{\left(i\right)}}\sqrt{\hat{v}_t^{\left(i\right)}}\left(\sqrt{v_t^{\left(i\right)}}+\sqrt{\hat{v}_t^{\left(i\right)}}\right)}\\
&\!=\! \left(1-\beta\right)\hat{\eta}_t^{\left(i\right)}g_t^{\left(i\right)}+ \hat{\eta}_t^{\left(i\right)}\sigma_t\frac{\left(1-\theta_t\right)g_t^{\left(i\right)}}{\sqrt{v_t^{\left(i\right)}}}\frac{\left(1-\beta\right)\sigma_t}{\sqrt{v_t^{\left(i\right)}}+\sqrt{\hat{v}_t^{\left(i\right)}}}\\
& \qquad \!-\! \hat{\eta}_t^{\left(i\right)}g_t^{\left(i\right)}\frac{\left(1-\theta_t\right)g_t^{\left(i\right)}}{\sqrt{v_t^{\left(i\right)}}}\frac{\left(1-\beta\right)g_t^{\left(i\right)}}{\sqrt{v_t^{\left(i\right)}}+\sqrt{\hat{v}_t^{\left(i\right)}}}.
\end{split}
\]}

In addition, we can obtain that

{\small\[
\begin{split}
&\beta\alpha_tm_{t-1}^{\left(i\right)}\left(\frac{1}{\sqrt{v_t^{\left(i\right)}}} - \frac{1}{\sqrt{\theta_tv_{t-1}^{\left(i\right)}}}\right)\\
%&= \beta\alpha_tm_{t-1}^{\left(i\right)}\frac{\left(1-\theta_t\right)\left(g_t^{\left(i\right)}\right)^2}{\sqrt{v_t^{\left(i\right)}}\sqrt{\theta_tv_{t-1}^{\left(i\right)}}\left(\sqrt{v_t^{\left(i\right)}}+\sqrt{\theta_tv_{t-1}^{\left(i\right)}}\right)}\\
%& =\beta\alpha_tm_{t-1}^{\left(i\right)}\frac{\left(1-\theta_t\right)\left(g_t^{\left(i\right)}\right)^2}{\sqrt{v_t^{\left(i\right)}}+\sqrt{\theta_tv_{t-1}^{\left(i\right)}}}\\
%&+ \beta\alpha_tm_{t-1}^{\left(i\right)} \frac{\left(1-\theta_t\right)\left(g_t^{\left(i\right)}\right)^2}{\sqrt{v_t^{\left(i\right)}}\left(\sqrt{v_t^{\left(i\right)}} + \sqrt{\theta_tv_{t-1}^{\left(i\right)}}\right)} \left(\frac{1}{\sqrt{\theta_tv_{t-1}^{\left(i\right)}}} - \frac{1}{\sqrt{\hat{v}_t^{\left(i\right)}}}\right)\\
& = \hat{\eta}_t^{\left(i\right)}g_t^{\left(i\right)}\frac{\left(1-\theta_t\right)g_t^{\left(i\right)}}{\sqrt{v_t^{\left(i\right)}}}\left(\frac{\beta m_{t-1}^{\left(i\right)}}{\sqrt{v_t^{\left(i\right)}} +\sqrt{\theta_t v_{t-1}^{\left(i\right)}}}\right)\\
&+ \frac{\beta\alpha_t m_{t-1}^{\left(i\right)}\left(1-\theta_t\right)\left(g_t^{\left(i\right)}\right)^2\sigma_t^2}{\sqrt{v_t^{\left(i\right)}}\sqrt{\hat{v}^{\left(i\right)}_t}\sqrt{\theta_tv_{t-1}^{\left(i\right)}}\left(\sqrt{v_t^{\left(i\right)}}+\sqrt{\theta_tv_{t-1}^{\left(i\right)}}\right)\left(\sqrt{\hat{v}_t^{\left(i\right)}}+\sqrt{\theta_tv_{t-1}^{\left(i\right)}}\right)}\\
%&  = \hat{\eta}_t^{\left(i\right)}g_t^{\left(i\right)}\frac{\left(1-\theta_t\right)g_t^{\left(i\right)}}{\sqrt{v_t^{\left(i\right)}}}\left(\frac{\beta m_{t-1}^{\left(i\right)}}{\sqrt{v_t^{\left(i\right)}} +\sqrt{\theta_t v_{t-1}^{\left(i\right)}}}\right) \\ 
%& + \hat{\eta}_t^{\left(i\right)}\sigma_t\frac{\left(1-\theta_t\right)g_t^{\left(i\right)}}{\sqrt{v_t^{\left(i\right)}}}\left(\frac{\beta m_{t-1}^{\left(i\right)}}{\sqrt{\theta_tv_{t-1}^{\left(i\right)}}} \frac{\sqrt{1-\theta_t}g_t^{\left(i\right)}}{\sqrt{v_t^{\left(i\right)}}+\sqrt{\theta_tv_{t-1}^{\left(i\right)}}} \frac{\sqrt{1-\theta_t}\sigma_t}{\sqrt{\hat{v}_t^{\left(i\right)}}+\sqrt{\theta_tv_{t-1}^{\left(i\right)}}} \right).
\end{split}
\]}

For convenience, we denote

{\small\[
\begin{split}
A_t^{\left(i\right)} &= \frac{\beta m_{t-1}^{\left(i\right)}}{\sqrt{v_t^{\left(i\right)}}+\sqrt{\theta_tv_{t-1}^{\left(i\right)}}} - \frac{\left(1-\beta\right)g_t^{\left(i\right)}}{\sqrt{v_t^{\left(i\right)}}+\sqrt{\hat{v}_t^{\left(i\right)}}},\\
B_t^{\left(i\right)} &= \left(\frac{\beta m_{t-1}^{\left(i\right)}}{\sqrt{\theta_t v_{t-1}^{\left(i\right)}}} \frac{\sqrt{1-\theta_t}g_t^{\left(i\right)}}{\sqrt{v_t^{\left(i\right)}} +\sqrt{\theta_tv_{t-1}^{\left(i\right)}}}\frac{\left(1-\theta_t\right)\sigma_t}{\sqrt{\hat{v}_t^{\left(i\right)}} + \sqrt{\theta_tv_{t-1}^{\left(i\right)}}}\right) \\
&+ \frac{\left(1-\beta\right)\sigma_t}{\sqrt{v_t^{\left(i\right)}}+\sqrt{\hat{v}_t^{\left(i\right)}}}.
\end{split}
\]

Then, for each $t>2$, we obtain
\[\small
\begin{split}
&\Delta_t^{\left(i\right)} - \frac{\beta\alpha_t}{\sqrt{\theta_t}\alpha_{t-1}}\Delta_{t-1}^{\left(i\right)} \\
&= -\frac{\alpha_tm_t^{\left(i\right)}}{\sqrt{v_t^{\left(i\right)}}} + \frac{\beta\alpha_tm_{t-1}^{\left(i\right)}}{\sqrt{\theta_tv_{t-1}^{\left(i\right)}}} = -\alpha_t\left(\frac{m_t^{\left(i\right)}}{\sqrt{v_t^{\left(i\right)}}} - \frac{\beta m_{t-1}^{\left(i\right)}}{\sqrt{\theta_tv_{t-1}^{\left(i\right)}}}\right)\\
& = -\frac{\left(1-\beta\right)\alpha_tg_t^{\left(i\right)}}{\sqrt{v_t^{\left(i\right)}}} - \beta\alpha_tm_{t-1}^{\left(i\right)}\left(\frac{1}{\sqrt{v_t^{\left(i\right)}}} - \frac{1}{\sqrt{\theta_tv_{t-1}^{\left(i\right)}}}\right)\\
& = -\left(1-\beta\right)\hat{\eta}_t^{\left(i\right)}g_t^{\left(i\right)} - \hat{\eta}_t^{\left(i\right)}g_t^{\left(i\right)}\frac{\left(1-\theta_t\right)g_t^{\left(i\right)}}{\sqrt{v_t^{\left(i\right)}}}A_t^{\left(i\right)}\\
&- \hat{\eta}_t^{\left(i\right)}\sigma_t\frac{\left(1-\theta_t\right)g_t^{\left(i\right)}}{\sqrt{v_t^{\left(i\right)}}}B_t^{\left(i\right)}. 
\end{split}
\]
}
By using the above equality, we can obtain an upper bound for  $\mathbb{E}\langle \nabla f\left(x_t\right), \Delta_t\rangle$ as follows:
{\small\[
\begin{split}
&\mathbb{E}\langle \nabla f\left(x_t\right),\Delta_t\rangle =  \frac{\beta\alpha_t}{\sqrt{\theta_t}\alpha_{t-1}}\mathbb{E}\langle \nabla f\left(x_t\right),\Delta_{t-1}\rangle\\
&\qquad- \frac{1}{N}\sum_{i=1}^N \mathbb{E}\langle \nabla f\left(x_t\right), \left(1-\beta\right)\hat{\eta}_t^{\left(i\right)} g_t^{\left(i\right)}\rangle \\
& \qquad - \frac{1}{N}\sum_{i=1}^N\mathbb{E}\left\langle \nabla f\left(x_t\right), \hat{\eta}_t^{\left(i\right)}g_t^{\left(i\right)}\frac{\left(1-\theta_t\right)g_t^{\left(i\right)}}{\sqrt{v_t^{\left(i\right)}}}A_t^{\left(i\right)} \right\rangle\\ &\qquad- \frac{1}{N} \sum_{i=1}^N \mathbb{E}\left \langle \nabla f\left(x_t\right), \hat{\eta}_t^{\left(i\right)}\sigma_t\frac{\left(1-\theta_t\right)g_i^{\left(i\right)}}{\sqrt{v_t^{\left(i\right)}}}B_t^{\left(i\right)}\right\rangle.
\end{split}
\]
}

For the first term, we obtain
{\small\[
\begin{split}
&\mathbb{E}\langle \nabla f\left(x_t\right),\Delta_{t-1}\rangle \\
&= \mathbb{E}\langle \nabla f\left(x_{t-1}\right),\Delta_{t-1}\rangle + \mathbb{E}\langle \nabla f\left(x_t\right)-\nabla f\left(x_{t-1}\right),\Delta_{t-1}\rangle\\
%&\leq \mathbb{E}\langle \nabla f\left(x_{t-1}\right),\Delta_{t-1}\rangle + \mathbb{E} \|\nabla f\left(x_t\right)-\nabla f\left(x_{t-1}\right)\|\|\Delta_{t-1}\|\\
&\leq \mathbb{E}\langle \nabla f\left(x_{t-1}\right),\Delta_{t-1}\rangle \\
&+ L\mathbb{E} \left\|Q_s\left(\frac{1}{N}\sum_{i=1}^N \mathcal{Q}_w\left(\Delta_{t-1}^{\left(i\right)}+e_{t-1}^{\left(i\right)}\right) +e_{t-1}\right)\right\|\|\Delta_{t-1}\|\\
%&\leq \mathbb{E}\langle \nabla f\left(x_{t-1}\right),\Delta_{t-1}\rangle + (2-\delta_s) L\mathbb{E} \left\|\frac{1}{N}\sum_{i=1}^N \mathcal{Q}_w\left(\Delta_{t-1}^{\left(i\right)}+e_{t-1}^{\left(i\right)}\right) +e_{t-1}\right\|\|\Delta_{t-1}\|\\
%&\leq \mathbb{E}\langle \nabla f\left(x_{t-1}\right),\Delta_{t-1}\rangle + (2-\delta_s)L\mathbb{E}\|e_{t-1}\|\|\Delta_{t-1}\|\\&\qquad + (2-\delta_w)(2-\delta_s)L\frac{1}{N}\sum_{i=1}^N \|\Delta_{t-1}^{\left(i\right)}\|\|\Delta_{t-1}\|+ (2-\delta_w)(2-\delta_s)L\mathbb{E}\frac{1}{N}\sum_{i=1}^N\|e_{t-1}^{\left(i\right)}\|\|\Delta_{t-1}\|\\
%&\leq \mathbb{E}\langle \nabla f\left(x_{t-1}\right),\Delta_{t-1}\rangle + \frac{(2-\delta_s)L}{N}\sum_{i=1}^N \mathbb{E}\|e_{t-1}\|\|\Delta_{t-1}^{\left(i\right)}\|
%\\&\qquad +  \frac{(2-\delta_s)(2-\delta_w)L}{N}\sum_{i=1}^N \mathbb{E}\|\Delta_{t-1}^{\left(i\right)}\|^2+ \frac{(2-\delta_s)(2-\delta_w) L}{N^2}\mathbb{E}\sum_{i=1}^N\sum_{j=1}^N\|e_{t-1}^{\left(i\right)}\|\|\Delta_{t-1}^{\left(j\right)}\|\\
%&=M_{t-1}.
\end{split}
\]}

Meanwhile, with the definition of $\mathcal{Q}_s$ and $\mathcal{Q}_w$, it holds that
{\small\[
\begin{split}
&\left\|Q_s\left(\frac{1}{N}\sum_{i=1}^N \mathcal{Q}_w\left(\Delta_{t-1}^{\left(i\right)}+e_{t-1}^{\left(i\right)}\right) +e_{t-1}\right)\right\|\\
&\leq \frac{2-\delta_s}{N}\sum_{i=1}^N \left\|\mathcal{Q}_w\left(\Delta_{t-1}^{\left(i\right)}+e_{t-1}^{\left(i\right)}\right)\right\| +(2-\delta_s)\left\|e_{t-1}\right\|\\
&\leq \frac{(2-\delta_s)(2-\delta_w)}{N}\sum_{i=1}^N\left\|\Delta_{t-1}^{\left(i\right)}\right\|+\left\|e_{t-1}^{\left(i\right)}\right\| +(2-\delta_s)\left\|e_{t-1}\right\|.
\end{split}
\]}

Combining with {\small $\|\Delta_{t-1}\|\leq \frac{1}{N} \sum_{i=1}^N\|\Delta_{t-1}^{(i)}\|$}, it holds that
{\small \[
\mathbb{E}\langle \nabla f\left(x_t\right),\Delta_{t-1}\rangle \leq M_{t-1}.
\]}

For the second term, we obtain
{\small \begin{align*}
&-\frac{1}{N}\sum_{i=1}^N\mathbb{E}\langle \nabla f\left(x_t\right),\left(1-\beta\right)\hat{\eta}_t^{\left(i\right)}g_t^{\left(i\right)}\rangle 
%&=  -\frac{1}{N}\sum_{i=1}^N\mathbb{E}\langle \nabla f\left(x_t\right),\left(1-\beta\right)\hat{\eta}_t^{\left(i\right)}\mathbb{E}_t\left(g_t^{\left(i\right)}\right)\rangle \\
 =-\frac{1-\beta}{N}\sum_{i=1}^N \|\nabla f\left(x_t\right)\|_{\hat{\eta}_t^{\left(i\right)}}^2.
\end{align*}
}

For the third term, we have 
{\small \[
\begin{split}
&-\mathbb{E}\left\langle \nabla f\left(x_t\right), \hat{\eta}_t^{\left(i\right)}g_t^{\left(i\right)} \frac{\left(1-\theta_t\right)g_t^{\left(i\right)}}{\sqrt{v_t^{\left(i\right)}}} A_t^{\left(i\right)} \right\rangle\\
 &\leq \mathbb{E} \left\langle \frac{\sqrt{\hat{\eta}_t^{\left(i\right)}}|\nabla f\left(x_t\right)||g_t^{\left(i\right)}|}{\sigma_t}, \frac{\sqrt{\hat{\eta}_t^{\left(i\right)}}\sigma_t|A_t^{\left(i\right)}|\left(1-\theta_t\right)|g_t^{\left(i\right)}|}{\sqrt{v_t^{\left(i\right)}}}\right\rangle\\
%&\leq \frac{1-\beta}{4}\mathbb{E}\left\|\frac{\sqrt{\hat{\eta}_t^{\left(i\right)}}|\nabla f\left(x_t\right)||g_t^{\left(i\right)}|}{\sigma_t}\right\|^2 + \frac{1}{1-\beta}\left\|\frac{\sqrt{\hat{\eta}_t^{\left(i\right)}}\sigma_t|A_t^{\left(i\right)}|\left(1-\theta_t\right)|g_t^{\left(i\right)}|}{\sqrt{v_t^{\left(i\right)}}}\right\|^2\\
&\leq \frac{1-\beta}{4}\mathbb{E}\|\nabla f\left(x_t\right)\|^2_{\hat{\eta}_t^{\left(i\right)}} + \frac{C_1G\alpha}{\theta}\mathbb{E}\left\|\frac{\sqrt{1-\theta_t}g_t^{\left(i\right)}}{\sqrt{v_t^{\left(i\right)}}}\right\|^2,
\end{split}
\]
}
where the inequality holds with the following equalities and inequalities:
{\small \[
\begin{split}
&\mathbb{E}\left\|\frac{\sqrt{\hat{\eta}_t^{\left(i\right)}}|\nabla f\left(x_t\right)||g_t^{\left(i\right)}|}{\sigma_t}\right\|^2 = \mathbb{E}\left\|\frac{\hat{\eta}_t^{\left(i\right)}\nabla f\left(x_t\right)^2\left(g_t^{\left(i\right)}\right)^2}{\sigma_t^2}\right\|_1\\
&= \mathbb{E} \|\hat{\eta}_t^{\left(i\right)}\nabla f\left(x_t\right)^2\|_1 = \mathbb{E}\|\nabla f\left(x_t\right)\|^2_{\hat{\eta}_t^{\left(i\right)}},
\end{split}
\]
}
{\small \[
\sqrt{\hat{\eta}_t^{\left(i\right)}}\sigma_t \leq \sqrt{\frac{\alpha_t\sigma_t^2}{\sqrt{\left(1-\theta_t\right)\sigma_t^2}}} \leq \sqrt{\frac{\alpha G}{\sqrt{\theta}}},
\]}
and 
{\small \[
\begin{split}
& |A_t^{\left(i\right)}| %\leq \frac{\beta|m_{t-1}^{\left(i\right)}|}{\sqrt{\theta_t v_{t-1}^{\left(i\right)}}} + \frac{\left(1-\beta\right)|g_t^{\left(i\right)}|}{\sqrt{v_t^{\left(i\right)}}}
\leq  \frac{1}{\sqrt{1-\theta_t}} \left(\frac{\beta}{\sqrt{\theta_t\left(1-\gamma\right)}} + 1-\beta\right).
\end{split}
\]
}
Let \small{$C_1\! = \!\left(\frac{\beta/\left(1-\beta\right)}{\sqrt{\left(1-\gamma\right)\theta_1}}+1\right)^2$}, we can obtain \small{$|A_t^{\left(i\right)}| \!\leq\! \frac{\sqrt{C_1}\left(1-\beta\right)}{\sqrt{1-\theta_t}}$}.

Then we obtain the upper bound of the third term:
{\small\[
\begin{split}
&-\frac{1}{N}\mathbb{E}\sum_{i=1}^N\left\langle \nabla f\left(x_t\right), \hat{\eta}_t^{\left(i\right)}g_t^{\left(i\right)}\frac{\left(1-\theta_t\right)g_t^{\left(i\right)}}{\sqrt{v_t^{\left(i\right)}}}A_t^{\left(i\right)}\right\rangle\\
&\leq \frac{1}{N}\sum_{i=1}^N \frac{1-\beta}{4}\mathbb{E}\|\nabla f\left(x_t\right)\|_{\hat{\eta}_t^{\left(i\right)}}^2 + \frac{C_1G\alpha}{\sqrt{\theta}}\mathbb{E}\left\|\frac{\sqrt{1-\theta_t}g_t^{\left(i\right)}}{\sqrt{v_t^{\left(i\right)}}}\right\|^2.
\end{split}
\]}

For the fourth term, by using similar inequalities we have
$\small{
|B_t^{\left(i\right)}| %\leq\! \left(\frac{\beta|m_{t-1}^{\left(i\right)}|}{\sqrt{\theta_tv_{t-1}^{\left(i\right)}}} \frac{\sqrt{1-\theta_t}|g_t^{\left(i\right)}|}{\sqrt{v_t^{\left(i\right)}}\!+\!\sqrt{\theta_t v_t^{\left(i\right)}}}\frac{\sqrt{1-\theta_t}\sigma_t}{\sqrt{\hat{v}_t^{\left(i\right)}}\!+\!\sqrt{\theta_t v_{t\!-\!1}^{\left(i\right)}}}\right)\\
%& \qquad \!+\! \frac{\left(1-\beta\right)\sigma_t}{\sqrt{v_t^{\left(i\right)}}\!\!+\! \sqrt{\hat{v}_t^{\left(i\right)}}} 
\leq \frac{\left(1-\beta\right)\sqrt{C_1}}{\sqrt{1-\theta_t}}.
}$

Then, we have
{\small\[
\begin{split}
&-\frac{1}{N} \sum_{i=1}^N \mathbb{E}\left \langle \nabla f\left(x_t\right), \hat{\eta}_t^{\left(i\right)}\sigma_t\frac{\left(1-\theta_t\right)g_i^{\left(i\right)}}{\sqrt{v_t^{\left(i\right)}}}B_t^{\left(i\right)}\right\rangle \\
&\leq \frac{1}{N}\sum_{i=1}^N \frac{1-\beta}{4}\mathbb{E}\|\nabla f\left(x_t\right)\|_{\hat{\eta}_t^{\left(i\right)}}^2 + \frac{C_1G\alpha}{\sqrt{\theta}}\mathbb{E}\left\|\frac{\sqrt{1-\theta_t}g_t^{\left(i\right)}}{\sqrt{v_t^{\left(i\right)}}}\right\|^2.
\end{split}
\] }

By defining 
{\small\[
\begin{split}
&N_t =  \mathbb{E} \left[\frac{(2-\delta_s)}{N}\left(\sum_{i=1}^N L\|e_{t}\|\|\Delta_{t}^{\left(i\right)}\| + (2-\delta_s) \|\Delta_t^{(i)}\|^2 \right)\right.\\
%&+\frac{(2-\delta_w)(2-\delta_s)L}{N}\sum_{i=1}^N \|\Delta_{t}^{\left(i\right)}\|^2\\
&+ \frac{(2-\delta_w)(2-\delta_s)L}{N^2}\sum_{i=1}^N\sum_{j=1}^N\|e_{t}^{\left(i\right)}\|\|\Delta_{t}^{\left(j\right)}\|\\
&\left.+ \frac{1}{N}\sum_{i=1}^N \frac{2C_1G\alpha}{\sqrt{\theta}}\left\|\frac{\sqrt{1-\theta_t}g_t^{\left(i\right)}}{\sqrt{v_t^{\left(i\right)}}}\right\|^2\right],
\end{split}
\]  }
then with induction, we can find the upper bound for $M_t$.  

With the above bound, we have 
{\small\[
\begin{split}
M_t&\leq \frac{\beta\alpha_t}{\sqrt{\theta_t}\alpha_{t-1}}M_{t-1}+ N_t - \frac{1-\beta}{2N}\mathbb{E}\sum_{i=1}^N\|\nabla f\left(x_t\right)\|_{\hat{\eta}_t^{\left(i\right)}}^2 \\
&\leq \frac{\beta\alpha_t}{\sqrt{\theta_t}\alpha_{t-1}}M_{t-1}+ N_t.
\end{split}
\] }

Besides, as for the bound of the first term $M_1$, we have
{\small \begin{equation}
    \label{lemma7-1}
    \begin{aligned}
    &\mathbb{E}\left\langle \nabla f\left(x_1\right), -\frac{\alpha_1 m_1^{\left(i\right)}}{\sqrt{v_1^{\left(i\right)}}}\right\rangle = \mathbb{E}\left\langle \nabla f\left(x_1\right), \frac{\alpha_1 \left(1-\beta\right)g_1^{\left(i\right)}}{\sqrt{v_1^{\left(i\right)}}}\right\rangle\\
    &\leq \mathbb{E}\langle \nabla f\left(x_t\right), \left(1-\beta\right)\hat{\eta}_1^{\left(i\right)}g_1^{\left(i\right)}\rangle \\
    &\qquad +\mathbb{E}\left\langle \nabla f\left(x_t\right), \hat{\eta}_1^{\left(i\right)}\sigma_1^{\left(i\right)}\frac{\left(1-\theta_1\right)g_1^{\left(i\right)}}{\sqrt{v_1^{\left(i\right)}}}\frac{\left(1-\beta\right)\sigma_1^{\left(i\right)}}{\sqrt{v_1^{\left(i\right)}}+\sqrt{\hat{v}_1^{\left(i\right)}}}\right\rangle\\
    &\qquad  - \mathbb{E}\left\langle \nabla f\left(x_t\right), \hat{\eta}_1^{\left(i\right)}g_1^{\left(i\right)}\frac{\left(1-\theta_1\right)g_1^{\left(i\right)}}{\sqrt{v_1^{\left(i\right)}}}\frac{\left(1-\beta\right)g_1^{\left(i\right)}}{\sqrt{v_1^{\left(i\right)}}+\sqrt{\hat{v}_1^{\left(i\right)}}} \right\rangle
    \end{aligned}
\end{equation}
}

Hence, with inequality \eqref{lemma7-1}, and definition of $M_1$, it holds that
{\small \[
\begin{split}
  &M_1
  %&= \sum_{i=1}^N \mathbb{E}\left\langle \nabla f\left(x_1\right), -\frac{\alpha_1 m_1^{\left(i\right)}}{\sqrt{v_1^{\left(i\right)}}}\right\rangle +  \frac{(2-\delta_w)(2-\delta_s)L}{N}\sum_{i=1}^N \mathbb{E}\|\Delta_{t}^{\left(i\right)}\|^2 \\
  \leq \sum_{i=1}^N \frac{2C_1G\alpha}{\sqrt{\theta}}\mathbb{E}\left\|\frac{\sqrt{1-\theta_1}g_1^{\left(i\right)}}{\sqrt{v_1^{\left(i\right)}}}\right\|^2 - \frac{1-\beta}{2}\mathbb{E}\|\nabla f\left(x_t\right)\|_{\hat{\eta}_t^{\left(i\right)}}^2 \\
  &\qquad + \frac{(2-\delta_w)(2-\delta_s)L}{N} \sum_{i=1}^N\mathbb{E}\|\Delta_1^{\left(i\right)}\|^2 \leq N_1. 
  \end{split}
\]}

Therefore, we can obtain
{\small\[
\begin{split}
M_t &\leq \sum_{k=1}^t \frac{\alpha_t\beta^{t-k}}{\alpha_k\sqrt{\theta'^{t-k}}}N_k - \frac{1-\beta}{2N}\sum_{i=1}^N \mathbb{E}\|\nabla f\left(x_t\right)\|_{\hat{\eta}_t^{\left(i\right)}}^2 \\
&\leq \sum_{k=1}^t \sqrt{\gamma}^{t-k}N_k - \frac{1-\beta}{2N}\sum_{i=1}^N\mathbb{E}\|\nabla f\left(x_t\right)\|_{\hat{\eta}_t^{\left(i\right)}}^2.
\end{split}
\]}

Then, by summing $M_t$ up, we obtain
{\small \begin{equation}
    \label{lemma72}
\begin{split}
\sum_{t=1}^T M_t &\leq \sum_{t=1}^T\sum_{k=1}^t\sqrt{\gamma}^{t-k} N_k -- \frac{1-\beta}{2N}\sum_{i=1}^N\mathbb{E}\|\nabla f\left(x_t\right)\|_{\hat{\eta}_t^{\left(i\right)}}^2\\
&\leq\frac{1}{\left(1-\sqrt{\gamma}\right)}\sum_{t=1}^T N_t - \frac{1-\beta}{2N}\sum_{i=1}^N\sum_{t=1}^T\mathbb{E}\|\nabla f\left(x_t\right)\|_{\hat{\eta}_t^{\left(i\right)}}^2.
\end{split}
\end{equation}
}

By combining lemma \ref{esum}, \ref{eesum}, \ref{sumdelta} and definition of $N_t$, we have
{\small \[
\begin{split}
&\sum_{t=1}^T N_t \leq \left(\frac{\delta_s'}{\delta_s} +\frac{(2-\delta_w)\delta_w'}{\delta_s\delta_w}\right)\frac{L}{N} \sum_{t=1}^T\sum_{i=1}^N\|\Delta_t^{\left(i\right)}\| \\
&+ \left(\!\frac{(2\!-\!\delta_s)(2\!-\!\delta_w)\alpha^2L}{\theta\left(1-\sqrt{\gamma}\right)^2\delta_s\delta_w}\! +\! \frac{2C_1G\alpha}{\sqrt{\theta}}\!\right)\! \frac{1}{N}\! \sum_{i=1}^N \sum_{t=1}^T\mathbb{E}\left\|\frac{\sqrt{1-\theta_t}g_t^{\left(i\right)}}{\sqrt{v_t^{\left(i\right)}}}\right\|^2.\\
\end{split}
\]}

Then, using lemma \ref{sumgv} and \ref{sumnormdelta} it holds that
{\small
\begin{equation}
\label{lemma73}
\begin{split}
&\sum_{t=1}^T N_t \leq  \left(\!\frac{(2\!-\!\delta_w)(2\!-\!\delta_s)\alpha^2L}{\theta\left(1-\sqrt{\gamma}\right)^2\delta_s\delta_w} + \frac{2C_1G\alpha}{\sqrt{\theta}}\!\right)\! d\left[\log\!\left(\!1\!+\!\frac{G^2}{\epsilon d}\!\right)\!+\!\frac{\theta}{1\!-\!\theta}\right]\\
&+ \!\left(\!\frac{\delta_s'}{\delta_s}\! +\!\frac{(2\!-\!\delta_w)\delta_w'}{\delta_s\delta_w}\!\right)\! L \sqrt{T}\! \sqrt{\frac{\alpha^2d}{\theta\left(1\!-\!\!\sqrt{\gamma}\right)^2}\left[\log\!\left(\!1\!+\!\frac{G^2}{\epsilon d}\!\right)\!+\!\frac{\theta}{1\!-\!\theta}\!\right]}.
\end{split}
\end{equation}
}

With the inequalities \eqref{lemma72}, \eqref{lemma73}, we obtain the desired result.
\end{proof}

\vspace{-0.2cm}
\begin{lemma}\label{multitau}
 Let $\tau$ be randomly chosen from $\{1,2,\cdots,T\}$ with equal probabilities $p_\tau = 1/T$. We have the following estimate:
  {\small\[
  \mathbb{E}[\|\nabla f\left(x_\tau\right)\|^2]\leq \frac{\sqrt{G^2+\epsilon d}}{\alpha\sqrt{T}N}\sum_{i=1}^N\mathbb{E}\left[\sum_{t=1}^T\|\nabla f\left(\left(x_t\right)\right)\|_{\hat{\eta_t}^{\left(i\right)}}^2\right].
  \]}
\end{lemma}
\begin{proof}
Note that for any $i=1,2,\cdots, N$, $\|\hat{v}_t^{\left(i\right)}\|_1 = \theta_t \|v_{t-1}^{\left(i\right)}\|_1 +\left(1-\theta_t\right) \|\sigma_t\|^2$ and $\|g_t^{\left(i\right)}\|\leq G$. 

Then, it is straightforward to prove $\|v_t^{\left(i\right)}\|_1 \leq G^2 + \epsilon d$. 

%Hence, we have $\|\hat{v}_t^{\left(i\right)} + \epsilon\|_1 \leq G^2 + \epsilon d$. 

Utilizing this inequality, we have
{\small\[
\begin{split}
     &\|\nabla f\left(x_t\right)\|^2 = \frac{\|\nabla f\left(x_t\right)\|^2}{\sqrt{\|\hat{v}_t^{\left(i\right)} \|_1}}\sqrt{\|\hat{v}_t^{\left(i\right)}\|_1}\\
     %&= \sqrt{\|\hat{v}_t^{\left(i\right)} \|_1}\sum_{k=1}^d\frac{|\nabla_k f\left(x_t\right)|^2}{\sqrt{\sum_{l=1}^d \hat{v}_{t,l}^{\left(i\right)}}}\\
     %\leq \sqrt{\|\hat{v}_t^{\left(i\right)}\|_1}\alpha_t^{-1}\sum_{k=1}^d \frac{\alpha_t}{\sqrt{\hat{v}^{\left(i\right)}_{t,k}}}|\nabla_k f\left(x_t\right)|^2 \\
     %&= \sqrt{\|\hat{v}_t^{\left(i\right)}\|_1}\alpha_t^{-1}\|\nabla f\left(x_t\right)\|_{\hat{\eta}_t^{\left(i\right)}}^2\\
     & \leq \sqrt{G^2+\epsilon d}\alpha_t^{-1}\|\nabla f\left(x_t\right)\|_{\hat{\eta}_t^{\left(i\right)}}^2= \frac{\sqrt{G^2 + \epsilon d}}{\alpha_T}\|\nabla f\left(x_t\right)\|^2_{\hat{\eta}_t^{\left(i\right)}}.
\end{split}
\]}
Then, by using the definition of $x_{\tau}$, we obtain
{\small\[
\begin{split}
     \mathbb{E}\left[\|\nabla f\left(x_\tau\right)\|^2\right] &= \frac{1}{T}\sum_{t=1}^T\mathbb{E}\left[\|\nabla f\left(x_t\right)\|^2\right]\\
     &\leq \frac{\sqrt{G^2 + \epsilon d}}{\alpha \sqrt{T}N} \mathbb{E}\left[\sum_{i=1}^N\sum_{t=1}^T \|\nabla f\left(x_t\right)\|^2_{\hat{\eta}_t^{\left(i\right)}}\right].
\end{split}
\]}
Thus, the desired result is obtained. 
\end{proof}

\begin{proof}[Proof of Theorem \ref{complexity-quantization}]
By using the gradient Lipschitz continuity of $f$, it holds that

{\small\[
\begin{split}
&\mathbb{E} f\left(\hat{x}_{t+1}\right) \leq \mathbb{E}\left[f\left(\hat{x}_t\right) +\langle \nabla f\left(\hat{x}_t\right), \Delta_t\rangle + \frac{L}{2}\|\Delta_t\|^2 \right] \\
& = \mathbb{E}\left[f\left(\hat{x}_t\right) +\langle \nabla f\left(x_t\right), \Delta_t\rangle +\langle \nabla f\left(\hat{x}_t\right) - \nabla f\left(\tilde{x}_t\right), \Delta_t\rangle \right.\\
&\qquad  +\langle \nabla f\left(\tilde{x}_t\right) - \nabla f\left(x_t\right)\!,\! \Delta_t\rangle + \frac{L}{2}\|\Delta_t\|^2 ]\\
%&\leq \mathbb{E}\left[f\left(\hat{x}_t\right)\! +\!\langle \nabla f\left(x_t\right), \Delta_t\rangle \!+\!L\|E_t\|\|\Delta_t\| \!+\!L\|e_t\|\|\Delta_t\| \!+\! \frac{L}{2}\|\Delta_t\|^2 \right]\\
&\leq \mathbb{E}\left[f\left(\hat{x}_t\right) +\langle \nabla f\left(x_t\right), \Delta_t\rangle  + \frac{(2\!-\!\delta_s)(2\!-\!\delta_w)L}{N}\sum_{i=1}^N \|\Delta_{t}^{\left(i\right)}\|^2 \right.\\
&\left. +\frac{(2\!-\!\delta_s)L}{N}\!\sum_{i=1}^N\! \|e_{t}\|\!\|\Delta_{t}^{\left(i\right)}\| \!+\! \frac{(2\!-\!\delta_w)(2\!-\!\delta_s)L}{N^2}\!\sum_{i=1}^N\sum_{j=1}^N\!\|e_{t}^{\left(i\right)}\|\!\|\Delta_{t}^{\left(j\right)}\|\right]\\
&=\mathbb{E}f\left(\hat{x}_t\right) +M_t.
\end{split}
\]}
For fixed T, by summing up $t$ from $1$ to $T$, we can obtain
{\small \[
\begin{split}
f^*&\leq \mathbb{E}[f\left(\tilde{x}_{T+1}\right)] \leq f\left(x_1\right) + \sum_{t=1}^T M_t\\
&\leq f\left(x_1\right) +  C_{2} - \frac{1-\beta}{2N}\sum_{t=1}^T\sum_{i=1}^N\mathbb{E}\|\nabla f\left(x_t\right)\|_{\hat{\eta}_t^{\left(i\right)}}^2 \\
&+ LC_3\alpha\sqrt{T}\left(\frac{\delta_s'}{\delta_s} +\frac{(2-\delta_s)\delta_w'}{\delta_s\delta_w}\right).
\end{split}
\]}

Then, using Lemma \ref{multitau}, we obtain
{\small \[
\begin{split}
&\mathbb{E} \|\nabla f\left(x_\tau^T\right)\|^2 \leq \frac{\sqrt{G^2\!+\!\epsilon d}}{\alpha\sqrt{T}N}\mathbb{E}\left[\sum_{i=1}^N\sum_{t=1}^T\|\nabla f\left(x_t\right)\|_{\hat{\eta}_t^{\left(i\right)}}^2\right]\\
%&\leq \frac{2\sqrt{G^2+\epsilon d}}{\left(1\!-\!\beta\right)\alpha \sqrt{T}}\left(f\left(x_1\right)-f^* \!+\! C_{2}\! +\! LC_3\alpha\sqrt{T}\!\left(\!\frac{\delta_s'}{\delta_s}\! +\!\frac{(2\!-\!\delta_w)\delta_w'}{\delta_s\delta_w}\!\right)\!\right)\\
& \leq \frac{2\sqrt{G^2+\epsilon d}}{\left(1-\beta\right)\alpha \sqrt{T}}\left(f\left(x_1\right)-f^* + C_{2} \right)\\
&\qquad+ \left(\frac{\delta_s'}{\delta_s} +\frac{(2-\delta_w)\delta_w'}{\delta_s\delta_w}\right)\frac{2L\sqrt{G^2+\epsilon d}C_3}{\left(1-\beta\right)}.
\end{split}
\]}

Therefore, by using the inequality that 
{\small \[\mathbb{E}\|\nabla f\left(x_\tau^T\right)\|^2\leq \varepsilon + \left(\frac{\delta_s'}{\delta_s} +\frac{(2-\delta_w)\delta_w'}{\delta_s\delta_w}\right)\frac{2L\sqrt{G^2+\epsilon d}C_3}{\left(1-\beta\right)},\]}
we obtain 
$\small {\frac{2\sqrt{G^2+\epsilon d}}{\left(1-\beta\right)\alpha \sqrt{T}}\left(f\left(x_1\right)-f^* + C_{2} \right) \leq \varepsilon,}$
which gives that 
{\small \[T\geq \frac{4\left(G^2+\epsilon d\right)}{\left(1-\beta\right)^2\alpha^2 \varepsilon^2}\left(f\left(x_1\right)-f^* + C_{2} \right)^2.\]}

Hence, the desired result is obtained.
\end{proof}

\begin{proof}[Proof of Corollary \ref{bit-complexity-quantization}]
When using Example \ref{quantization-mapping2}, we can calculate $\delta = 1/2$ and $\delta' = 2^{k}\sqrt{d}$, when $max_{x \in \mathbb{X}} \|x\|_\infty \leq 2^{K} + 2^{K-1}$.

By solving $\left(\frac{\delta_s'}{\delta_s} +\frac{(2-\delta_w)\delta_w'}{\delta_s\delta_w}\right)\frac{2L\sqrt{G^2+\epsilon d}C_3}{\left(1-\beta\right)}\leq \frac{\varepsilon}{2}$, and setting $K$ to be the same in both directions, we obtain that when $k \leq -\log\left(\frac{24L\sqrt{d}\sqrt{G^2\!+\!\epsilon d}C_3}{\varepsilon(1-\beta)}\right)$, this equality always holds.

Besides, as it is shown in Lemma \ref{Lemmamv}, $\|\delta_t\| \leq\alpha/((1-\gamma)\theta)$.
With assumption that $\|Q(x)\| \leq 3/2 \|x\|$, when $K \geq \log\left(\alpha/((1-\gamma)\theta)\right)$ we can get a quantizer satisfying Definition \ref{quantization-compressor}.

Besides, using Theorem \ref{complexity-quantization}, with inequality 
{\small\[
\mathbb{E}\|\nabla f\left(x_\tau^T\right)\|^2 \leq \varepsilon/2 + \left(\frac{\delta_s'}{\delta_s} +\frac{(2-\delta_w)\delta_w'}{\delta_s\delta_w}\right)\frac{2L\sqrt{G^2+\epsilon d}C_3}{\left(1-\beta\right)},\]}
we have {\small $T\geq \frac{16\left(G^2+\epsilon d\right)}{\left(1-\beta\right)^2\alpha^2 \varepsilon^2}\left(f\left(x_1\right)-f^* + C_{4} \right)^2$.}

Meanwhile, we transmit $\small{d\log(K-k+1)}$ bits per iteration.
\end{proof}

\vspace{-0.2cm}
\begin{proof}{Proof of Corollary \ref{compressor-complexity}}

Because when compressor satisfies Definition \ref{compressor-def}, where $\|Q(x) - x\| \leq (1-\delta)\|x\|$, it holds that $\|Q(x) \| \leq (2-\delta)\|x\|$. Therefore, by setting $\delta' = 0$, we obtain the desired result.
\end{proof}

% \section{Proof of Proposition \ref{quantization-def-check}}

% \begin{proof}{Proof of Example 1}

% By the definition of 32-bits floating number, it can be shown when $x \leq 3.4 \times 10^{38}$, $|Q_1(x) - x|\leq 2^{-23}|x| + 2^{-126}$.

% Besides $\|Q_2(x) - x\|\leq \sqrt{\frac{d-1}{d-1+2^k}}||x||$ for all $\|x\|_\infty \in [1-2^{-23},\frac{1}{1-2^{-23}}]$.

% Then, $\|Q(x) - x\| \leq \sqrt{\frac{d-1}{d-1+2^k}} \|x\| + 2^{-126}\sqrt{d}$ and $\|Q(x)\| \leq \left(1+\sqrt{\frac{d-1}{d-1+2^k}}\right) \|x\|$.

% \end{proof}

% \begin{proof}{Proof of Example 2}

% When $2^{k-1} \leq |x_i| \leq 2^{K} + 2^{K-1}$, we have $\min_{y \in M} | x_i - y| \leq \frac{1}{2}|x_i|$.

% When $|x_i|  \leq 2^{k}$, it directly holds that $\min{y\in M}|x_i - y| \leq 2^{k}$.

% Therefore, $\|Q(x) - x\| \leq \frac{1}{2}\|x\| + 2^{k}\sqrt{d}$ and $\|Q(x)\| \leq \frac{3}{2} \|x\|$.

% \end{proof}

\end{document}